\documentclass[11pt]{article}

\usepackage[usenames,dvipsnames]{xcolor}
\definecolor{Gred}{RGB}{219, 50, 54}
\definecolor{ToCgreen}{RGB}{0, 128, 0}

\usepackage{hyperref}
\hypersetup{
  colorlinks=true,
  citecolor=ToCgreen,
  linkcolor=Sepia,
  filecolor=Gred,
  urlcolor=Gred
  }

\usepackage{graphicx} 
\usepackage[margin=1in]{geometry}

\usepackage{amsmath, amsthm}
\usepackage{mathtools}
\mathtoolsset{showonlyrefs}
\usepackage{verbatim}
\usepackage{hyperref} 
\usepackage{natbib} 
\usepackage{hyperref}       
\usepackage{url}            
\usepackage{booktabs}       
\usepackage{nicefrac}       
\usepackage{microtype}      
\usepackage{multicol}
\usepackage{enumitem}
\usepackage{mathtools}
\usepackage{thm-restate}
\newtheorem{definition}{Definition}

\newtheorem{remark}{Remark}
\newtheorem{theorem}{Theorem}
\newtheorem{assumption}{Assumption}
\newtheorem{example}{Example}

\newtheorem{lemma}[theorem]{Lemma}
\newcommand{\TV}{\mathrm{TV}}
\newcommand{\KL}{\mathrm{KL}}
\newcommand{\Hellinger}{\mathrm{H}}
\newcommand{\LC}{\mathrm{LC}}
\newcommand{\Wtwo}{\mathrm{W}_2}
\newcommand{\sigmamin}{\sigma^2_{\mathrm{min}}}
\newcommand{\sigmamax}{\sigma^2_{\mathrm{max}}}
\newcommand{\SLC}{\mathrm{SLC}}
\newcommand{\subG}{\mathrm{subG}}
\newcommand{\Sinit}{S_{\mathrm{init}}}
\newcommand{\Send}{S_{\mathrm{target}}}
\newcommand{\wh}{\widehat}

\usepackage[T1]{fontenc}

\usepackage[bitstream-charter,cal=cmcal]{mathdesign}

\title{Critical windows: non-asymptotic theory for feature emergence in diffusion models}
\author{
    Marvin Li\thanks{Email:
    \texttt{marvinli@college.harvard.edu}} \\
    Harvard College
    \and
    Sitan Chen\thanks{Email: \texttt{sitan@seas.harvard.edu}, supported in part by NSF Award 2331831} \\
    Harvard SEAS
}

\newcommand{\shortleft}{\scriptscriptstyle\leftarrow}
\newcommand{\R}{\mathbb{R}}
\newcommand{\D}{\mathrm{d}}

\newcommand{\forward}{X}
\newcommand{\reverse}{X^{\shortleft}}
\newcommand{\revsub}[1]{X^{\shortleft,#1}}
\newcommand{\pathrev}[1]{P^{\shortleft,#1}}
\newcommand{\modrev}[3]{X^{\shortleft}_{#3}[#1^{\langle #2\rangle}]}
\newcommand{\modrevpath}[2]{P^{\shortleft}[#1^{\langle #2\rangle}]}

\newcommand{\norm}[1]{\|#1\|}

\newcommand{\Tupper}{T_{\mathrm{upper}}}
\newcommand{\Tlower}{T_{\mathrm{lower}}}

\usepackage{ifthen}
\newcommand{\modrevlaw}[3]{\ifthenelse{\equal{#3}{}}{p[#1^{\langle #2\rangle}]}{p_{T-#3}[#1^{\langle #2\rangle}]}}

\begin{document}
\maketitle

\begin{abstract}

We develop theory to understand an intriguing property of diffusion models for image generation that we term \emph{critical windows}. Empirically, it has been observed that there are narrow time intervals in sampling during which particular features of the final image emerge, e.g. the image class or background color~\citep{ho2020denoising,meng2022sdedit,9879163,raya2023spontaneous,georgiev2023journey,sclocchi2024phase,biroli2024dynamical}. While this is advantageous for interpretability as it implies one can localize properties of the generation to a small segment of the trajectory, it seems at odds with the continuous nature of the diffusion. 

We propose a formal framework for studying these windows and show that for data coming from a mixture of strongly log-concave densities, these windows can be provably bounded in terms of certain measures of inter- and intra-group separation. We also instantiate these bounds for concrete examples like well-conditioned Gaussian mixtures. Finally, we use our bounds to give a rigorous interpretation of diffusion models as hierarchical samplers that progressively ``decide'' output features over a discrete sequence of times. 

We validate our bounds with synthetic experiments. Additionally, preliminary experiments on Stable Diffusion suggest critical windows may serve as a useful tool for diagnosing fairness and privacy violations in real-world diffusion models.
\end{abstract}

\section{Introduction}

Diffusion models currently stand as the predominant approach to generative modeling in audio and image domains~\citep{sohl2015deep,dhariwal2021diffusion,song2020score,ho2020denoising}. At their core is a ``forward process'' that transforms data into noise, and a learned ``reverse process'' that progressively undoes this noise, thus generating fresh samples. Recently, a series of works has established rigorous convergence guarantees for diffusion models for arbitrary data distributions~\citep{DBLP:conf/iclr/ChenC0LSZ23,lee2023convergence,chen2023improved,benton2023linear}. While these results prove that in some sense diffusion models are entirely principled, the generality with which they apply suggests further theory is needed to explain the rich behaviors of diffusion models specific to the \emph{real-world} distributions on which they are trained.

In this work, we focus on a phenomenon that we term \emph{critical windows}. In the context of image generation, it has been observed that there are narrow time intervals along the reverse process during which certain features of the final image are determined, e.g. the class, color, background~\citep{ho2020denoising,meng2022sdedit,9879163,raya2023spontaneous,georgiev2023journey,sclocchi2024phase,biroli2024dynamical}. This suggests that even though the reverse process operates in continuous time, there is a series of discrete ``jumps'' during the sampling process during which the model ``decides'' on certain aspects of the output. The existence of these critical windows is highly convenient from an interpretability standpoint, as it lets one zoom in on specific parts of the diffusion model trajectory to understand how some feature of the generated output emerged. 

Despite the strong empirical evidence for the existence of critical windows (e.g. the striking Figures 3, B.6, and B.10 from~\citet{georgiev2023journey} and Figures 1 and 2 from~\cite{sclocchi2024phase}), our mathematical understanding of critical windows is very immature. Indeed, from the perspective of prior theory,\footnote{See Section~\ref{sec:related} for a discussion of concurrent works.} the different times of the reverse process largely behave as equal-class citizens, outside the realm of very simple toy models of data. We thus ask:
\begin{center}
	\emph{Can we \textbf{prove} the existence of critical windows in the reverse process for a rich family of data distributions?}
\end{center}
Before stating our theoretical findings, we outline the framework we adopt (see Section~\ref{sec:noisedenoise} for a formal treatment). Also, as issues of discretization, score error, and the support of the data distribution lying on a lower-dimensional submanifold are orthogonal to this paper, throughout we will conflate the data distribution with the output distribution of the model and assume the reverse process is run in continuous time with perfect score.

\subsection{General framework} 
\label{sec:framework_intro}

\begin{figure}
    \centering
    \includegraphics[width=\textwidth]{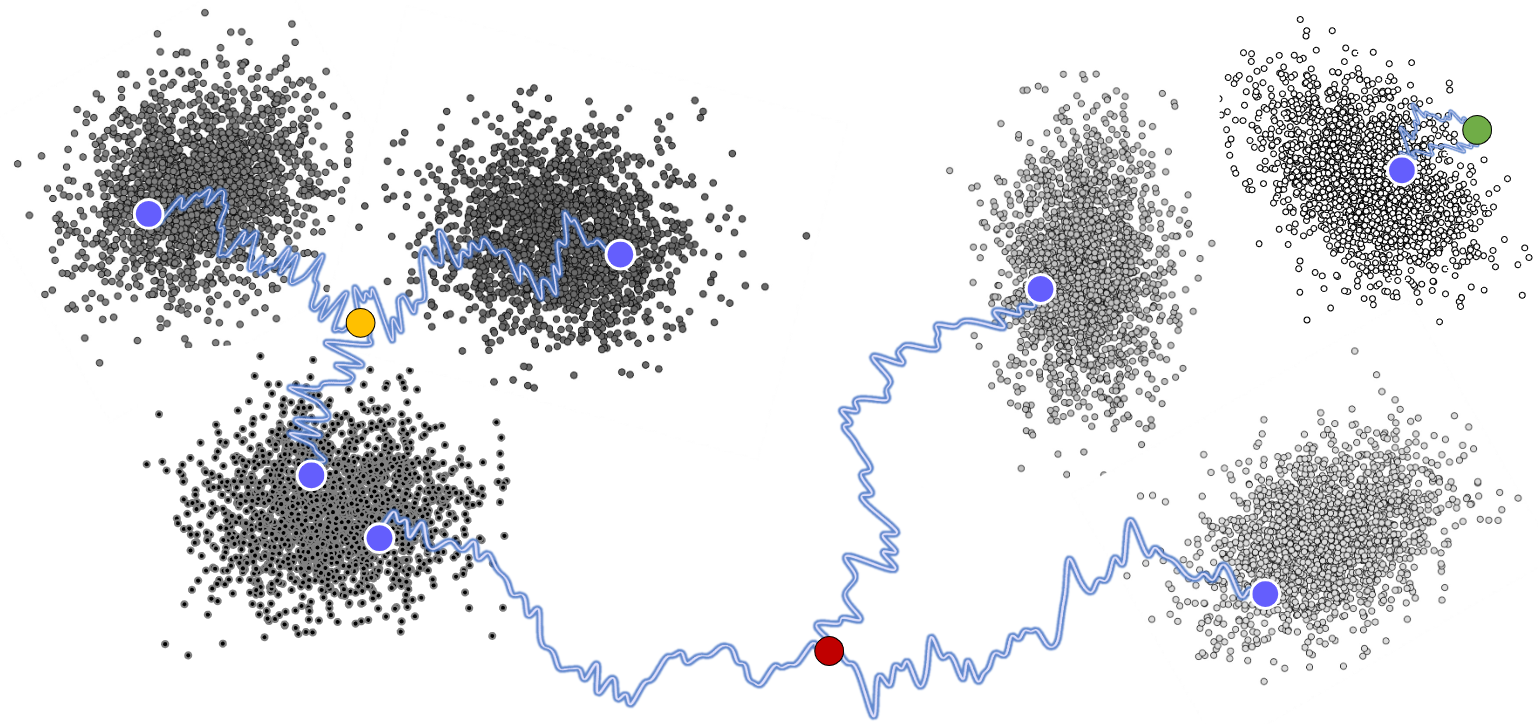}
    \caption{Cartoon depiction of running forward process for time $t$ (to produce one of the red/yellow/green dots corresponding to large/medium/small $t$) and then running reverse process (trajectories in blue) for time $t$ to sample from some sub-mixture. \vspace{-1em}}
    \label{fig:cartoon}
\end{figure}

Our starting point is a setup related to one in~\citet{georgiev2023journey} and also explored in the concurrent work of~\cite{sclocchi2024phase} \--- see Section~\ref{sec:related} for a comparison to these two works. Given a sample $x$ from the data distribution $p$, consider the following experiment. We run the forward process (see Eq.~\eqref{eq:forward} below) starting from $x$ for intermediate amount of time $t$ to produce a noisy sample $x_t$. We then run the reverse process (see Eq.~\eqref{eq:reverse}) for time $t$ starting from $x_t$ to produce a new sample $x'$ (see Section~\ref{sec:prelims} for formal definitions). Observe that as $t\to\infty$, the distribution over $x_t$ converges to Gaussian, and thus the resulting distribution over $x'$ converges to $p$. As $t\to 0$, the distribution over $x'$ converges to a point mass at $x$ \--- in this latter regime, it was empirically observed by~\cite{ho2020denoising} that for small $t$, the distribution over $x'$ is essentially given by randomly modifying low-level features of $x$.

\paragraph{Critical windows for mixture models.} Qualitatively, we can ask for the first, i.e. largest, time $t$ for which samples from the distribution over $x'$ mostly share a certain feature with $x$. To model this, we consider a data distribution $p$ given by a mixture of sub-populations $p^1,\ldots,p^K$. A natural way to quantify whether $x'$ shares a feature with $x$ is then to ask whether the distribution over $x'$ is close to a particular \emph{sub-mixture}. E.g., if $x$ is a cat image, and $\Send\subset\{1,\ldots,K\}$ denotes the sub-populations $p^i$ corresponding to cat images, then one can ask whether there is a critical window of $t$ such that the distribution over $x'$ is close to the sub-mixture given by $\Send$ (see Figure~\ref{fig:cartoon}). 

Finally, rather than reason about specific initial samples from $p$, we will instead marginalize out the randomness of $x$ so we can reason at a more ``population'' level. Concretely, we consider $x$ which is drawn from some $p^i$ for $i\in\Send$, or more generally from some sub-mixture indexed by a subset $\Sinit\subset\Send$, and consider the resulting marginal distribution over $x'$, which we denote by $\modrevlaw{\Sinit}{t}{}$. So if for instance $\Sinit$ denoted the sub-mixture of \emph{brown} cats, then if there is a critical window of times for which $\modrevlaw{\Sinit}{t}{}$ is close to the sub-mixture given by $\Send$, then one can interpret these times as the point at which, to generate a brown cat, the diffusion model ``decides'' its sample will be a cat.

\subsection{Our contributions}

Our results are threefold: (1) we give a general characterization of the critical window for a rich family of multimodal distributions, (2) we specialize these bounds to specific distribution classes to get closed-form predictions, (3) we use these to prove, under a distributional assumption, that the reverse process is a ``hierarchical sampler'' that makes a series of discrete feature choices to generate the output.

\paragraph{General characterization of critical window.} 
We consider distributions $p$ which are \emph{mixtures of strongly log-concave distributions} in $\R^d$. In Section~\ref{sec:master}, we give general bounds on the the critical window at which $\modrevlaw{\Sinit}{t}{}$ approximates the sub-mixture given by $\Send$ for any choice of $\Sinit$.  These bounds depend on the total variation (TV) distance between sub-populations inside and outside $\Sinit$ and $\Send$ along the forward process. We identify two endpoints (see Eqs.~\eqref{eq:Tlower} and~\eqref{eq:Tupper} for formal definitions):
\begin{itemize}[itemsep=0pt,leftmargin=*,topsep=0pt]
    \item $\Tlower$: the time in the forward process at which the initial sub-mixture indexed by $\Sinit$ and the target sub-mixture indexed by $\Send$ first become close in TV
    \item $\Tupper$: the time in the forward process at which a component in $\Send$ begins to exhibit non-negligible overlap with a component in the rest of the mixture\footnote{\emph{A priori} $T_{\mathrm{lower}}$ need not be smaller than $T_{\mathrm{upper}}$ In Section~\ref{sec:instantiate}, we show this holds when $\Send$ corresponds to a ``salient'' feature.}
\end{itemize}

\begin{theorem}[Informal, see Theorem~\ref{masters_theorem}]\label{thm:masters_informal}
    Suppose $p$ is a mixture of strongly log-concave distributions, and let $\Sinit \subset \Send$. For any $t\in[\Tlower,\Tupper]$, if one runs the forward process for time $t$ starting from the sub-mixture given by $\Sinit$, then runs the reverse process for time $t$, the result will be close in TV to the sub-mixture given by $\Send$.
\end{theorem}
\noindent As we show empirically on synthetic examples (Fig.~\ref{fig:gaussian_hierarchy}) these bounds can be highly predictive of the true critical windows.

The intuition for this result is that there are two competing effects at work. On the one hand, if $t$ is sufficiently large, then running the forward process for time $t$ starting from either the initial sub-mixture given by $\Sinit$ versus the target sub-mixture given by $\Send$ will give rise to similar distributions. So if we run the reverse process on these, the resulting distributions will remain close, thus motivating our definition of $\Tlower$. On the other hand, we want these resulting distributions to be close to the target sub-mixture indexed by $\Send$. But if $t$ is too large, they will merely be close to $p$. To avoid this, we need $t$ to be small enough that along the reverse process, the overall score function of $p$ to remains close to the score function of the target sub-mixture. Intuitively, this should happen provided the components of $p$ outside $\Send$ do not overlap much with the ones inside $\Send$ even after running the forward process for time $t$, thus motivating our definition of $\Tupper$.


\begin{remark}
    Going back to our original motivation of capturing heterogeneous distributions like the distribution over natural images as mixture models, it may at first glance seem extremely strong that we require that each sub-population form a strongly log-concave component. e.g. in the latent space over which the diffusion operates. We clarify that in our setting, a sub-population might correspond to a \emph{sub-mixture} consisting of multiple such strongly log-concave components. In the example of natural images, if we think of a particular image class as a sub-mixture consisting of neighborhoods around different images in the embedding space, then because intuitively the neighborhood around any particular image is compact in the embedding space, it should be approximately unimodal, justifying our strong log-concavity assumption.
\end{remark}

\paragraph{Concrete estimates for critical times.}

The endpoints of the critical window in Theorem~\ref{thm:masters_informal} are somewhat abstract. Our second contribution is to provide concrete bounds for these\--- see Section~\ref{sec:examples} for details. We first consider the general setting of Theorem~\ref{thm:masters_informal} where $p$ is a mixture of strongly log-concave distributions, under the additional assumption that the components would be somewhat close in Wasserstein distance if they were shifted to all have mean zero.

\begin{theorem}[Informal, see Theorem~\ref{corr:wassersteinlower}]\label{thm:subgaussian_informal}
    Suppose $p$ is a mixture of $1/\sigma^2$-strongly log-concave distributions with means $\mu_1,\ldots,\mu_K$, and let $\Sinit\subset \Send$.
    
    Suppose that for any $i\in\Send$ and any $j\not\in\Send$, $\norm{\mu_i - \mu_j} \gtrsim \sigma\sqrt{d}$. Then there is an upper bound for $\Tlower$, which is dominated by $\ln\max_{i\in\Sinit, j\in\Send} \norm{\mu_i - \mu_j}$, and there is a lower bound for $\Tupper$, which is dominated by $\ln\min_{i\in\Send, j\not\in \Send}\norm{\mu_i - \mu_j}$.
\end{theorem}

\noindent Theorem~\ref{thm:subgaussian_informal} shows that the start time of the critical window scales as the log of the max distance between any component in $\Sinit$ and any component in $\Send$, whereas the end time scales as the log of the min distance between any component in $\Send$ and any component in $[K]\backslash\Send$. We can interpret this as saying the following about feature emergence. If $\Send$ corresponds to the part of the data distribution with some particular feature, if that feature is sufficiently \emph{salient} in the sense that typical images with that feature are closer to each other, then $\Tlower < \Tupper$, and therefore there exists a critical window of times $t$ during which the feature associated to $\Send$ emerges. For latent diffusion models in particular, the manifold of images in latent space becomes highly structured, so there will be “salient” features such that images with the same “salient” feature will be closer together in the latent space. Furthermore, the length of this window, i.e. the amount of time after the features associated to $\Send$ emerge but before other features do, is logarithmic in the ratio between the level of separation between $\Send$ and $[K]\backslash \Send$, versus the level of separation within $\Send$. 
\begin{remark}
If the target mixture is the same as the initial mixture, $\Tlower =0$ and we only need $\Tupper>0$ to form a critical window. This setting is especially useful for interpretability and data attribution, which usually  examines an object x with property p and asks for the largest time for which property p is preserved.
\end{remark}

In Appendix~\ref{sec:dictionary}, we specialize the bound in Theorem~\ref{thm:subgaussian_informal} to a sparse coding setting where the means of the components are given by sparse linear combinations of a collection of incoherent ``dictionary vectors.'' In this setting, we show that the endpoints $\Tlower$ (resp. $\Tupper$) have a natural interpretation in terms of the \emph{Hamming distances} between the sparse linear combinations defining the means within $\Sinit$ and $\Send$ (resp. between $\Send$ and $[K]\backslash\Send$).

Theorem~\ref{thm:subgaussian_informal} is quite general except for one caveat: we must assume that the the components outside of $\Send$ have some level of separation. Note that a $1/\sigma^2$-strongly log-concave distribution in $d$ dimensions will mostly be supported on a thin shell of radius $\sigma\sqrt{d}$ \citep{kannan1995isoperimetric}, so our assumption essentially amounts to ensuring the balls that these shells enclose, for any component inside $\Send$ and any component outside $\Send$, do not intersect.

Next, we remove this caveat for mixtures of Gaussians:
\begin{theorem}[Informal, see Theorem~\ref{thm::well_conditioned_gaussian_theorem}]\label{thm:gaussian_informal}
    Suppose $p$ is a mixture of $K$ identity-covariance Gaussians in $\R^d$with means $\mu_1,\ldots,\mu_K$, and let $\Sinit\subset \Send$. Then there is an upper bound for $\Tlower$, which is dominated by $\ln\max_{i\in\Sinit, j\in\Send} \norm{\mu_i - \mu_j}$, and there is a lower bound for $\Tupper$, which is dominated by $\ln\min_{i\in\Send, j\not\in \Send}\norm{\mu_i - \mu_j}$.
\end{theorem}

\noindent In fact our result extends to general mixtures of Gaussians with sufficiently well-conditioned covariances, see Theorem~\ref{thm::well_conditioned_gaussian_theorem}. We also explore the dependence on the mixing weights of the different components (see Appendix~\ref{app:weights}).

\paragraph{Hierarchical sampling interpretation.}

Thus far we have focused on a specific target sub-mixture $\Send$, which would correspond to a specific feature in the generated output. In Section~\ref{sec:hierarchy}, we extend these findings to distributions with a \emph{hierarchy} of features. To model this, we consider Gaussian mixtures with hierarchical clustering structure. This structure ensures the mixture decomposes into well-separated clusters of components such that the separation between clusters exceeds the separation within clusters, and furthermore each cluster recursively satisfies the property of being decomposable into well-separated clusters, etc. This naturally defines a tree, which we call a \emph{mixture tree}, where each node of the tree corresponds to a cluster at some resolution, with the root corresponding to the entire data distribution and the leaves corresponding to the $K$ individual components of the mixture (see Definition~\ref{def:tree}).

If we think of every node $v$ as being associated with a feature, then the corresponding cluster of components is comprised of all sub-populations which possess that feature, in addition to all features associated to nodes on the path from the root to $v$. By chaining together several applications of Theorem~\ref{thm:gaussian_informal}, we prove the following:
\begin{theorem}[Informal, see Theorem~\ref{thm:hierarchy_example}] \label{thm:hierarchy_informal}
    For a hierarchical mixture of identity-covariance Gaussians with means specified by a mixture tree, for any root-to-leaf path $(v_0,\ldots,v_L)$ in the mixture tree, where the leaf $v_L$ corresponds to a component $p^i$ of the mixture, there exists an $\underline{L}$ and a discrete sequence of times $t_{v_{\underline{L}}}  > \ldots > t_{v_L}$ such that for all $\underline{L} \le \ell \le L$, the distribution if one runs the forward process for time $t_{v_\ell}$ starting from the sub-mixture given by the node $v_L$ and the reverse process for time $t_{v_\ell}$, the result will be close in TV to the sub-mixture given by node $v_\ell$.
\end{theorem}

\noindent This formalizes the intuition that to sample from distributions with this hierarchical structure, the sampler makes a discrete sequence of choices on the features to include. This discrete sequence of choices corresponds to a whittling away of other components from the score until the sampler reaches the end component. Through adding larger scales of noise, contributions to the score from increasingly distant classes are incorporated into the reverse process.

\section{Related work}
\label{sec:related}

\paragraph{Comparison to~\cite{georgiev2023journey}.} \citet{georgiev2023journey} empirically studied a variant of critical windows in the context of data attribution. For a generated image $x_0$ given by some trajectory $\{x_t\}_{t\in [0,T]}$ of the reverse process, they consider rerunning the reverse process starting at some intermediate point $x_t$ in the trajectory (they refer to this as sampling from the ``conditional distribution,'' which they denote by $p(\cdot\mid x_t)$). They then compute the probability that the images sampled in this fashion share a given feature with $x_0$ and identify critical times $T^{\rm cond}_{\rm lower}<T^{\rm cond}_{\rm upper}$ such that sampling from $p(\cdot |x_{T^{\rm cond}_{\rm lower}})$ preserves the given feature in the original image while sampling from $p(\cdot |x_{T^{\rm cond}_{\rm upper}})$ does not. Our definition is different: instead of rerunning the reverse process, we run the forward process for time $t$ starting from $x_0$ to produce $x^{\rm cond}_t$ and then run the reverse process from $x^{\rm cond}_t$ to sample from $p(\cdot|x^{\rm cond}_t)$. 
Note that in our definition, even after the initial generation $\{x_t\}$ is fixed, there is still randomness in $x^{\rm cond}_t$. This means that unlike the setting in~\citet{georgiev2023journey}, our setup is meaningful even if the reverse process is deterministic, e.g. based on an ODE. Additionally, our setup is arguably more flexible for data attribution as it does not require knowledge of the \emph{trajectory} $\{x_t\}_{t\in [0,T]}$ that generated $x_0$. In general, we expect that our critical window thresholds are less than \citet{georgiev2023journey}'s thresholds because adding noise to the state at intermediate times could also change the features. We view our theoretical contributions as complementary to their empirical work in rigorously understanding qualitatively similar phenomena and also use CLIP for our own experiments.

\paragraph{Comparison to~\cite{raya2023spontaneous}.} To our knowledge, the most relevant prior theoretical work is that of \citet{raya2023spontaneous} (we also discuss two recent works \citep{sclocchi2024phase,biroli2024dynamical} concurrent with ours below). Along the reverse process, they consider a \emph{fixed-point path} over which the reverse process incurs zero drift and argue that diffusion models exhibit a phase transition at the point where the spectrum of the Hessian of the potential bifurcates into positive and negative parts. They then give an end-to-end asymptotic analysis of this for the special case of a discrete distribution supported on two points, and some partial results for more general discrete distributions. In contrast, we give end-to-end guarantees for a more general family of high-dimensional distributions, but under a different perspective than the one of stability of fixed points considered by~\citet{raya2023spontaneous}. We find it quite interesting that one can understand critical windows through such different theoretical lenses.

On the empirical side, they also conducted various real-data experiments; we give a detailed comparison between our experimental setup and theirs later in this section.

\paragraph{Concurrent works.} Next, we discuss the relation between our work and concurrent works by~\citet{sclocchi2024phase} and~\citet{biroli2024dynamical} that also studied the critical window phenomenon. As with the results of~\citet{raya2023spontaneous}, we view all of these works as offering complementary mathematical insights into critical windows; here we highlight some key differences.

\citet{sclocchi2024phase} considered the same setting of running the forward process for some time $t$ starting from a sample and then running the reverse process, which they refer to as ``forward-backward experiments.''  Instead of mixture models, they consider a very different data model with hierarchical structure, the \emph{random hierarchy model}~\citep{petrini2023deep}, which is a \emph{discrete} distribution over one-hot embeddings of strings in some multi-level context-free grammar. Mathematically, this amounts to a distribution over \emph{standard basis vectors} where the probabilities encode some hierarchical structure.

They observe numerically that if one runs the reverse diffusion with exact score estimation (using belief propagation), there appears to be a critical window. They then give accurate but non-rigorous statistical physics-based predictions for the location of this window by passing to a certain mean-field approximation. In contrast, we studied \emph{mixture models}, where the notion of hierarchical structure is encoded geometrically into the locations of the components. Additionally, in our setting, we provide fully rigorous bounds on the locations of critical windows.


The theoretical setting of~\citet{biroli2024dynamical} is closer to that of the present work. They study a mixture of \emph{two} spherical Gaussians under the ``conditional sampling'' setting of~\cite{georgiev2023journey} rather than our noising and denoising framework. Relevant to this work, they identify a phase transition that they call ``speciation'' which roughly corresponds to the critical time in the reverse process at which the trajectory starts specializing to one of the two components. Using the same kind of Landau-type perturbative calculation used to predict second-order phase transitions in statistical physics, the authors give highly precise but non-rigorous asymptotic predictions for the time at which speciation occurs. In contrast, our work studies a more general data model but provides less precise but non-asymptotic and rigorous estimates for the critical window. Interestingly,~\citet{biroli2024dynamical} also suggest a useful heuristic based on the time at which the noise obscures the principal component of the data distribution and validate this heuristic on real data (see below). In our mixture model setting, this is closely related to the separation between components and thus suggests ties from their theory and numerics to ours.

These works also conduct experiments on real data. Below, we elaborate on how our experiments differ from theirs.

\paragraph{Critical window experiments on real data.} Numerous studies have investigated the critical window phenomenon in diffusion models \citep{ho2020denoising,raya2023spontaneous,biroli2024dynamical,sclocchi2024phase}. These papers demonstrate a dramatic jump in the similarity of some feature within a narrow time range, either under our noising and denoising framework \citep{sclocchi2024phase} or the ''conditional sampling'' framework \citep{georgiev2023journey,raya2023spontaneous,biroli2024dynamical} described above. The main distinguishing factors between these different experiments are the diffusion models tested, the varying definitions of what a ``feature'' entails, and the method to determine whether a given image has a certain feature. \citet{raya2023spontaneous,georgiev2023journey,sclocchi2024phase,biroli2024dynamical} identify the critical windows of class membership for unconditional diffusion models operating in pixel space that were trained on small, hand-labeled datasets like MNIST or CIFAR-10. \citet{biroli2024dynamical} were able to obtain precise predictions for the critical times for a simple diffusion model trained on two classes. \citet{raya2023spontaneous,georgiev2023journey,biroli2024dynamical} trained supervised classifiers to sort image generations into different categories, whereas \citet{sclocchi2024phase} employed the hidden layer activations of an ImageNet classifier to define high- and low-level features of an image and computed the cosine similarity of the embeddings of the base and new image. Among all these empirical results, our experimental setup most closely mirrors Figure B.10 of ~\citet{georgiev2023journey}; we both experiment with StableDiffusion 2.1, manually inspect the image for potential features, and use CLIP instead of a supervised classifier to label images into different categories. That said, recall from the discussion at the beginning of this section that this paper and~\citet{georgiev2023journey}'s experiments examine different critical window frameworks (noise and denoise vs. conditional sampling).

\paragraph{Theory for diffusion models.} Recently several works have proven convergence guarantees for diffusion models~\citep{debetal2021scorebased,BloMroRak22genmodel,chen2022improved, DeB22diffusion, leelutan22sgmpoly, liu2022let, Pid22sgm, WibYan22sgm, DBLP:conf/iclr/ChenC0LSZ23, chen2023restoration, lee2023convergence,li2023towards,benton2023error,chen2023probability,li2024towards}. Roughly speaking, these results show that diffusion models can sample from essentially any distribution over $\R^d$, assuming access to a sufficiently accurate estimate for the score function. Our work is orthogonal to these results as they focus on showing that diffusion models can be used to sample. In contrast, we take for granted that we have access to a diffusion model that can sample; our focus is on specific properties of the sampling process. That said, there are isolated technical overlaps, for instance the use of path-based analysis via Girsanov's theorem, similar to~\cite{DBLP:conf/iclr/ChenC0LSZ23}.

\paragraph{Mixtures of Gaussians and score-based methods.} 

Gaussian mixtures have served as a fruitful testbed for the theory of score-based methods. In~\cite{shah2023learning}, the authors analyzed a gradient-based algorithm for learning the score function for a mixture of spherical Gaussians from samples and connected the training dynamics to existing algorithms for Gaussian mixture learning like EM. In~\cite{cui2023analysis}, the authors gave a precise analysis of the training dynamics and sampling behavior for mixtures of two well-separated Gaussians using tools from statistical physics. Other works have also studied related methods like Langevin Monte Carlo and tempered variants~\cite{koehler2023sampling, lee2018beyond} for learning/sampling from Gaussian mixtures. We do not study the learnability of Gaussian mixtures. Instead, we assume access to the true score and try to understand specific properties of the reverse process.

\section{Technical preliminaries}
\label{sec:prelims}

\subsection{Probability and diffusion basics}

\paragraph{Probability notation.}
We consider the following divergences and metrics for probability measures. Given distributions $P,Q$, we use $\TV(P,Q) \triangleq \frac{1}{2}\int |dP-dQ| d\mu$ to denote the \emph{total variation distance}, $\LC(P,Q) \triangleq \frac{1}{2}\int \frac{(dP-dQ)^2}{d(P+Q)} d\mu$ to denote the \textit{Le Cam distance}, $\mathrm{H}^2(P,Q) \triangleq \int(\sqrt{dP}-\sqrt{dQ})^2 d\mu$ to denote the squared \emph{Hellinger distance}, and $\Wtwo(P,Q) \triangleq \sqrt{\inf_{\gamma \sim \Gamma(P,Q)}\mathbb{E}_{(x,y)\sim \gamma}\|x-y\|^2}$, where $\Gamma(P,Q)$ is the set of all couplings between $P,Q$, to denote the \textit{Wasserstein-$2$ distance}. We use the following basic relation among these quantities, a proof of which we include in Appendix~\ref{app:ratio_inequality} for completeness.
\begin{restatable}{lemma}{ratioinequalitylemma}\label{ratio_inequality}
For probability measures $P,Q$, 
$$\mathbb{E}_{x \sim P} \Bigl[\frac{\D Q}{\D P+\D Q} \Bigr]
    = \frac{1}{2}(1-\LC(P,Q)) 
    \le \frac{1}{2}(1-\frac{1}{2}\mathrm{H}^2(P,Q)) 
    \le \frac{1}{2}\sqrt{1-\TV^2(P,Q)}.$$
\end{restatable}

\noindent Let $\subG_d(\sigma^2)$ denote the class of sub-Gaussian random vectors in $\R^d$ with variance proxy $\sigma^2$. Let $\SLC(\beta,d)$ denote the set of $1/\beta$-strongly log-concave distributions over $\R^d$.

\paragraph{Diffusion model basics.} Let $q$ be a distribution over $\R^d$ with smooth density. In diffusion models, there is a \emph{forward process} which progressively transforms samples from $q$ into pure noise, and a reverse process which undoes this process. For the former, we consider the Ornstein-Uhlenbeck process for simplicity.
This is a stochastic process $(\forward_t)_{t\ge 0}$ given by the stochastic differential equation (SDE)
\begin{equation}
    \D\forward_t = -\forward_t\,\D t + \sqrt{2}\D B_t\,, \qquad \forward_0 \sim q\,, \label{eq:forward}
\end{equation}
where $(B_t)_{t\ge 0}$ is a standard Brownian motion. Given $t\ge 0$, let $q_t \triangleq \mathrm{law}(\forward_t)$, so as $t\to \infty$, $q_t$ converges exponentially quickly to the standard Gaussian distribution $\gamma^d$.

Let $T\ge 0$ denote a choice of terminal time for the forward process. For the reverse process, denoted by $(\reverse_t)_{t\in[0,T]}$, we consider the standard reverse SDE given by
\begin{equation}
    \D\reverse_t = \{\reverse_t + 2\nabla\ln q_{T-t}(\reverse_t)\}\,\D t + \sqrt{2}\,\D B_t \label{eq:reverse}
\end{equation}
for $\reverse_0 \sim q_T$, where here $(B_t)_{t\ge 0}$ is the reversed Brownian motion. The most important property of the reverse process is that $q_{T-t}$ is precisely the law of $\reverse_t$. 

\paragraph{Girsanov's theorem.}  The following is implicit in an approximation argument due to~\cite{DBLP:conf/iclr/ChenC0LSZ23} which is applied in conjunction with Girsanov's theorem. This lets us compare the path measures of the solutions to two SDEs with the same initialization:

\begin{theorem}[Section 5.2 of~\cite{DBLP:conf/iclr/ChenC0LSZ23}]\label{thm:girsanov}
    Let $(Y_t)_{t\in[0,T]}$ and $(Y'_t)_{t\in[0,T]}$ denote the solutions to
    \begin{align*}
        \D Y_t &= b_t(Y_t) \, \D t + \sqrt{2}\D B_t\,, \qquad Y_0 \sim q \\
        \D Y'_t &= b'_t(Y'_t)\, \D t + \sqrt{2}\D B_t\,, \qquad Y'_0 \sim q\,.
    \end{align*}
    Let $q$ and $q'$ denote the laws of $Y_T$ and $Y'_T$ respectively. If $b_t, b'_t$ satisfy that $\int^T_0 \mathbb{E}_Q\,\norm{b_t(Y_t) - b'_t(Y_t)}^2\, \D t < \infty$, then $\KL(q\|q') \le \int^T_0 \mathbb{E}_Q\,\norm{b_t(Y_t) - b'_t(Y_t)}^2\,\D t$.
\end{theorem}

\subsection{Main framework: noising and denoising mixtures}
\label{sec:noisedenoise}

We will consider data distributions $p$ given by \emph{mixture models}. For component distributions $p^1,\ldots,p^K$ over $\R^d$ and mixing weights $w_1,\ldots,w_K$ summing to $1$, let $p \triangleq \sum_i w_i p^i$. Let $\mu_i$ denote the mean of $p^i$. For any nonempty $S\subset[K]$, we define the \emph{sub-mixture} $p^S$ by $p^S \triangleq \sum_{i\in S} \frac{w_i}{\sum_{j\in S} w_j} p^i$.
Let $(\forward^S_t)_{t\in[0,T]}$ denote the forward process given by running Eq.~\eqref{eq:forward} with $q = p^S$, let $p^S_t$ denote the law of $\forward^S_t$, and let $(\revsub{S}_t)$ denote the reverse process given by running Eq.~\eqref{eq:reverse} with $q = p^S$. When $S = \{i\}$, we drop the braces in the superscripts. Given intermediate time $\wh{T} \in [0,T]$, we denote the path measure for $(\revsub{S}_t)_{t\in[0,\wh{T}]}$ by $\pathrev{S}_{\wh{T}} \in\mathcal{C}([0,\wh{T}], \R^d)$.

\paragraph{The targeted reverse process.} The central object of study in this work is a modification of the reverse process for the overall mixture $p$ in which the initialization is changed from $p_T$ to an \emph{intermediate point} in the forward process \emph{for a sub-mixture}. Concretely, given $\wh{T}\in[0,T]$ and nonempty $S\subset[K]$, define the modified reverse process $(\modrev{S}{\wh{T}}{t})_{t\in[0,\wh{T}]}$ to be given by running the reverse SDE in Eq.~\eqref{eq:reverse} with $q = p$, with terminal time $\wh{T}$ instead of $T$, and initialized at $p^S_{\wh{T}}$ instead of $p_{\wh{T}}$. We denote the law of $\modrev{S}{\wh{T}}{t}$ by $\modrevlaw{S}{\wh{T}}{t}$ and the path measure for $(\modrev{S}{\wh{T}}{t})_{t\in[0,\wh{T}]}$ by $\modrevpath{S}{\wh{T}}\in\mathcal{C}([0,\wh{T}], \R^d)$. When $t = T$, we omit the subscript in the former.

 \begin{enumerate}[noitemsep,topsep=0em,leftmargin=*]
    \item Draw a sample $X$ from the sub-mixture $p^S$
    \item Run forward process for time $\wh{T}$ from $X$ to produce $X'$
    \item From terminal time $\wh{T}$, run the reverse process starting from $X'$ for time $t$ to produce $\modrev{S}{\wh{T}}{t}$
\end{enumerate}

Because this process reverses the forward process conditioned on a particular subset $S$ of the original mixture components, we refer to $(\modrev{S}{\wh{T}}{t})_{t\in[0,\wh{T}]}$ as the \emph{$S$-targeted reverse process} from noise level $\wh{T}$. We caution that the $S$-targeted reverse process should not be confused with the standard reverse process where the data distribution is taken to be $p^S$, as the score function being used in the targeted process is that of the full mixture $p$ rather than that of $p^S$.

\paragraph{Mixture model parameters.}

We consider the following quantities for a given mixture model, which characterize levels of separation within and across subsets of the mixture. Given $S,S'\subset[K]$, define
\begin{align*}
    \overline{R}\triangleq \max_{i \in [K]}\|\mu_i\| \qquad w(S,S')\triangleq \max_{i\in S,j \in S'}\|\mu_i-\mu_j\|\\
    \Delta(S)\triangleq \min_{\ell \in S,j \in [K]-S} \|\mu_\ell -\mu_j \| \qquad \overline{W} \triangleq \max_{i,j\in[K]} \frac{w_i}{w_j}\,.
\end{align*}
Lastly, we characterize the level of imbalance across sub-populations via $\overline{W} \triangleq \max_{i,j\in[K]} \frac{w_i}{w_j}$.

\section{Master theorem for critical times}
\label{sec:master}

Recall that $\Sinit\subset \Send\subset[K]$ denote the two sub-mixtures we are interested in. In the notation of Section~\ref{sec:noisedenoise}, we wish to establish upper and lower bounds on the time $\wh{T}$ at which
\begin{equation}
    \TV(\modrevlaw{\Sinit}{\wh{T}}{}, p^{\Send}) \label{eq:targetclose}
\end{equation}
becomes small.

Given error parameter $0 < \epsilon < 1$, define
\begin{align}
T_{\mathrm{lower}}(\epsilon) \triangleq \inf \{&t \in [0,T]: \TV(p_{t}^{S_{\mathrm{init}}  },p_{t}^{\Send }) \leq \epsilon\} \label{eq:Tlower}  \\
T_{\mathrm{upper}}(\epsilon) \triangleq \sup \{&t \in [0,T]: \TV(p^i_t,p^j_t) \geq 1-\epsilon^2/2 \nonumber \\
&\forall i \in \Send, j \in [K]-\Send\}\,. \label{eq:Tupper}
\end{align}
When $\epsilon$ is clear from context, we refer to these times as $T_{\mathrm{lower}}$ and $T_{\mathrm{upper}}$. Based on the intuition above, we expect that Eq.~\eqref{eq:targetclose} is small provided $\wh{T} \ge T_{\mathrm{lower}}$ and $\wh{T} \le T_{\mathrm{upper}}$. In this section, we prove that this is indeed the case for any $p$ given by a mixture of strongly log-concave distributions (see Remark~\ref{remark:slc} for discussion on the assumption of strong log-concavity of components).

\begin{assumption}[Strong log-concavity]\label{slc_components}
    For some $\Psi^2 \geq 1$, $p^i \in \SLC(\Psi^2,d)$.
\end{assumption}
\begin{assumption}[Smooth components]\label{lipschitz_score}
    For some $L > 0$ and for all $t \geq 0$, the score $\nabla \ln p^i_t$ is $L$-Lipschitz.
\end{assumption}
\begin{assumption}[Moment bound] \label{higher_moments} 
    For some $M \ge 1$ and for all $i \in [K]$ and $t\in [0,T]$, $\mathbb{E}\,\|\forward^i_t\|^4 \leq M$.
\end{assumption}
\noindent Finally, our bounds will depend on how large the score for any component is over samples from any other component:
\begin{assumption}[Score bound]\label{score_bound_assumption}
    For some $\overline{M} \ge 0$ and for all $i,j \in [K]$, $t \in [0,T]$, $\mathbb{E}_{X \sim p_t^j}\|\nabla \ln p^i_t(X)\|^4 \leq \overline{M}$.
\end{assumption}
\noindent We compute $\overline{M}$ for various examples in Section~\ref{sec:examples}, but for now one can safely think of $\overline{M}$ as scaling polynomially in the dimension and in the parameter $\overline{R}$.

\begin{remark}\label{remark:slc}
    It turns out that the only place where we need strong log-concavity of the components in the mixture is in the rather technical estimate of Lemma~\ref{diff_score_expo}, which is also much stronger than what is necessary for Theorem~\ref{masters_theorem}. It suffices to show the LHS of Lemma~\ref{diff_score_expo} integrates to a finite value. While we only prove the bound in that Lemma rigorously for strongly log-concave components, we expect it to hold even for more general families of non-log-concave distributions.  
\end{remark}

\subsection{Main result and proof sketch} We are now ready to state our main bound for the critical time $\wh{T}$ at which Eq.~\eqref{eq:targetclose} becomes small.

\begin{theorem}\label{masters_theorem} 
Let $S_{\mathrm{init}}\subset \Send \subset [K]$. For $\epsilon > 0$, if $\wh{T} \ge T_{\mathrm{lower}}(\epsilon)$ and $\wh{T} \le T_{\mathrm{upper}}(\epsilon)$, then  
\begin{equation}
\TV(\modrevlaw{\Sinit}{\wh{T}}{}, p^{\Send}) \lesssim \epsilon \sqrt{\overline{W}}K^{2}\bigl(\overline{R}^2+M^2+\sqrt{M}\Psi^4+\sqrt{\overline{M}}\bigr)\,.
\end{equation}
\end{theorem}

The proof of Theorem~\ref{masters_theorem} relies on the following technical lemma whose proof we defer to Appendix~\ref{app:diff_score_expo}.

\begin{restatable}{lemma}{fourthscoredifflemma} \label{diff_score_expo}
Under Assumptions~\ref{slc_components}, \ref{higher_moments}, and~\ref{score_bound_assumption}, $\mathbb{E}_{X \sim p^i_t} \|\nabla \ln p^j_t(X)-\nabla \ln p^\ell_t(X)\|^4 \lesssim  e^{-4t}(\overline{R}^4+M^4+M\Psi^8+\overline{M}) \ \ \forall i,j,\ell\in[K]$.
\end{restatable}

\noindent Informally, this lemma quantities the extent to which the score functions for $p^j$ and $p^\ell$ become close over the course of the forward process, as measured by an average sample from any other component of the mixture.
\begin{proof}[Proof of Theorem~\ref{masters_theorem}]
By data processing inequality and definition of $T_{\mathrm{lower}}, T_{\mathrm{upper}}$, for all $i \in \Send, j \not\in \Send$,
\begin{align}
\TV(p_{t}^{S_{\mathrm{init}}  },p_{t}^{\Send }) \leq \epsilon \qquad &\forall t \in [\widehat{T},T] \label{eq:useTlower} \\
\TV(p^i_t,p^j_t) \geq 1-\epsilon^2/2 \qquad &\forall t \in [0,\widehat{T}]\,,  \label{eq:useTupper}
\end{align}
By data processing inequality and triangle inequality,
\begin{align*}
    &\TV(\modrevlaw{\Sinit}{\wh{T}}{},p^{\Send})\leq \TV(\modrevpath{\Sinit}{\wh{T}}, \pathrev{\Send}_{\wh{T}}) \le \underbrace{\TV(\modrevpath{\Sinit}{\wh{T}}, \modrevpath{\Send}{\wh{T}})}_{\text{(I)}}+\underbrace{\TV(\modrevpath{\Send}{\wh{T}}, \pathrev{\Send}_{\wh{T}})}_{\text{(II)}}\,.
\end{align*}
As $\modrevpath{\Sinit}{\wh{T}}$ and $\modrevpath{\Send}{\wh{T}}$ are the path measures for the solutions to the same SDE with initializations $p^{\Sinit}_{\wh{T}}$ and $p^{\Send}_{\wh{T}}$ respectively, we can use data processing again to bound (I) via
\begin{align}
\TV(\modrevpath{\Sinit}{\wh{T}},\modrevpath{\Send}{\wh{T}})\leq \TV(p^{S_{\mathrm{init}}  }_{\widehat{T}},p^{\Send}_{\widehat{T}}) \leq \epsilon.
\end{align}
To bound (II), we apply Pinsker's and Theorem~\ref{thm:girsanov} to bound $\TV(\modrevpath{\Send}{\wh{T}}, \pathrev{\Send}_{\wh{T}})^2$ by
\begin{equation*}
\TV(\modrevpath{\Send}{\wh{T}}, \pathrev{\Send}_{\wh{T}})^2 \leq \int_0^{\widehat{T}}\mathbb{E}\,\|\nabla \ln p_t(\forward^{\Send}_t)-\nabla \ln p_t^{\Send}(\forward^{\Send}_t)\|^2\,\D t\,.
\end{equation*}
We have the following identity (see Appendix~\ref{app:scorediffrewrite} for proof):
\begin{restatable}{lemma}{scorediffrewrite}\label{lem:scorediff_rewrite}
$\|\nabla \ln p^{\Send}_t -\nabla \ln p_t\bigr\|^2 = \bigl\|\nabla \ln p_t^{\Send} - \nabla \ln p_t^{[K]-\Send} \bigr\|^2 \cdot \left( \frac{\sum_{i \in [K]-\Send} w_i p^i_t}{\sum_{i\in[K]} w_i p^i_t}\right)^2$.
\end{restatable} 
\noindent Using this expression, we can invoke Cauchy-Schwarz to separate the two terms that appear on the right-hand side. We bound these two terms in turn.
Recalling the definition of $\overline{W}$ and also applying Lemma \ref{ratio_inequality}, we see that for any $j\in\Send$,
\begin{align*}
  &\mathbb{E}_{p^j_t}\Bigl( \frac{\sum_{i \in [K]\backslash\Send} w_i p^i_t}{\sum_{i\in[K]} w_i p^i_t}\Bigr)^4 \leq \sum_{\ell \in [K]\backslash\Send}\mathbb{E}_{p^j_t}\Bigl[\frac{ w_\ell p^\ell_t}{w_j p^j_t + w_\ell p^\ell_t}\Bigr]\lesssim K \overline{W}\max_{\ell \in [K]\backslash\Send}\sqrt{1-\TV^2(p_t^\ell,p_t^j)} \lesssim  K \overline{W}\epsilon^2\,,
\end{align*}
where in the last step we used Eq.~\eqref{eq:useTlower}. By convexity, the same bound thus holds when the expectation on the left-hand side is replaced by an expectation with respect to $p^{\Send}_t$.

By the same convexity argument, to bound $\mathbb{E}\,\bigl\|\nabla \ln p_t^{\Send}(\forward_t^{\Send}) - \nabla \ln p_t^{[K]-\Send}(\forward_t^{\Send})\bigr\|^4$, it suffices to show that the expectations
\begin{align}
 \mathbb{E}\,\bigl\|\nabla \ln p_t^{\Send}(\forward_t^i) - \nabla \ln p_t^{[K]-\Send}(\forward_t^i)\bigr\|^4   
\end{align}
for all $i \in [K]$ are bounded. Moreover, the score of a mixture is a weighted average of the scores of the components, $\nabla \ln p_t^{\Send} = \sum_{i \in \Send} \frac{w_i p^i_t}{\sum_{j \in \Send} w_j p^j_t} \nabla \ln p_t^i$. By the triangle inequality, $\|\nabla \ln p_t^{\Send}(\forward_t^i) - \nabla \ln p_t^{[K]-\Send}(\forward_t^i)\|$ is at most the difference between two elements of a weighted score. Thus, we have 
\begin{align*}
\mathbb{E}\,\bigl\|\nabla \ln p_t^{\Send}(\forward_t^{\Send}) - \nabla \ln p_t^{[K]-\Send}(\forward_t^{\Send})\bigr\|^4 &\le \mathbb{E}_i \mathbb{E}\max_{   \substack{i,j \in \Send\\\ell \in [K]-\Send}} \bigl\|\nabla \ln p^j_t(\forward^i_t) - \nabla \ln p^\ell_t(\forward^i_t)\bigr\|^4 \\
&\le K^3\max_{   \substack{i,j \in \Send\\\ell \in [K]-\Send}}\mathbb{E}\,\bigl\|\nabla \ln p^j_t(\forward_t^i) - \nabla \ln p^\ell_t(\forward_t^i)\bigr\|^4. 
\end{align*}
Thus we can conclude by applying Lemma \ref{diff_score_expo} and bound $\mathbb{E}\,\bigl\|\nabla \ln p^{\Send}_t(\forward_t^{\Send}) -\nabla \ln p_t(\forward_t^{\Send})\bigr\|^2$ by $O(\epsilon \sqrt{\overline{W}}K^{2}(\overline{R}^2+M^2+\sqrt{M}\Psi^4+\sqrt{\overline{M}})e^{-2t})$.
Integrating over $[0,\wh{T}]$ completes the proof.
\end{proof}
\section{Instantiating the master theorem}
\label{sec:instantiate}

We now consider cases where we can provide concrete bounds on $T_{\mathrm{lower}}, T_{\mathrm{upper}}$. Our bounds here hold independent of the Assumptions in Section~\ref{sec:master}.

\subsection{General mixtures with similar components}

We first consider the case where the components of the mixture are ``similar'' in the sense that if we take any two components and translate them to both have mean zero, then they are moderately close in Wasserstein distance. Here, we obtain the following bounds on $T_{\rm lower}$ and $T_{\rm upper}$:

\begin{restatable}{lemma}{wassersteinlower}\label{corr:wassersteinlower}
Let $\epsilon > 0$. For $i \in [K]$, let $\overline{p}^i$ denote the density of the $i$-th component of the mixture model $p$ after being shifted to have mean zero. Suppose $\Wtwo(\overline{p}^i,\overline{p}^j) \leq \Upsilon$ for all $i,j \in [K]$. Then $T_{\rm lower}(\epsilon) \le \Bigl\{ \ln (w(S_{\rm init},\Send)+\Upsilon) +\ln \frac{1}{\epsilon}+\frac{1}{2} \ln 2 \Bigr\}\vee \, 3$.
Additionally, if $p^i_0 \in \subG_d(\sigma^2)$ for all $i\in[K]$, then
$T_{\rm upper}(\epsilon) \ge \ln \Delta(\Send) - \ln \sigma - \ln\sqrt{8d \ln 6 + 8\ln 4/\epsilon^2} - \ln 3-\frac{1}{2}\ln 8$.
\end{restatable}

\begin{proof}[Proof sketch of Lemma~\ref{corr:wassersteinlower}, see Appendix~\ref{app:w2}] For $T_{\rm lower}$, we apply Pinsker's inequality and a Wasserstein smoothing to upper bound the $\TV$ between components in the initial and target mixture in terms of the Wasserstein-$2$ distance of the components, which decreases at the rate of $O(e^{-t}(w(S_{\rm init},\Send)+\Upsilon))$. For $T_{\rm upper}$, we use sub-Gaussian concentration bounds to lower bound the $\TV$ between components in $S_{\rm end}$ and $[K]-S_{\rm end}$. 
\end{proof}
\noindent Note that because all $\alpha$-strongly log-concave distributions are sub-Gaussian with variance proxy $\Theta(1/\alpha)$, under Assumption~\ref{slc_components} of Section~\ref{sec:master} the above applies for $\sigma \asymp \Psi$.

When the terms $\Upsilon, \Psi, 1/\epsilon$ are sufficiently small, our bounds on $T_{\mathrm{lower}}$ and $T_{\mathrm{upper}}$ are dominated by $\ln w(\Sinit, \Send)$ and $\ln \Delta(\Send)$ respectively. Recall that $w(\Sinit, \Send)$ and $\Delta(\Send)$ respectively correspond to the maximum distance between any two component means from $\Sinit$ and $\Send$, and the minimum distance from $\Send$ to the rest of the mixture. This, combined with our master theorem, has the favorable interpretation that as long as the separation between components within $\Sinit$ and $\Send$ is dominated by the separation between components in $\Send$ vs. outside $\Send$, then there is a non-empty window of times $\wh{T} \in [T_{\rm lower}, T_{\rm upper}]$ such that the $\Sinit$-targeted reverse process from noise level $\wh{T}$ results in samples close to $\Send$.

\subsection{Mixtures of well-conditioned Gaussians}
\label{sec:examples}
We now suppose $p$ is a mixture of Gaussians, with $p^i = \mathcal{N}(\mu_i,\Sigma_i)$. At time $t \geq 0$ in the forward process, if $\mu_i(t) \triangleq e^{-t}\mu_i$ and $\Sigma_i(t)\triangleq e^{-2t}\Sigma_i+(1-e^{-2t})\mathrm{Id}$, then $p^i_t = \mathcal{N}(\mu_i(t), \Sigma_i(t))$, $p_t = \sum w_i  \mathcal{N}(\mu_i(t), \Sigma_i(t))$.

We also define $\sigmamax(t):= \max_i \sigmamax(\Sigma_i(t))$, $\sigmamin(t) = \min_i \sigmamin(\Sigma_i(t))$, and $\overline{R}(t) = e^{-t}\max_{i} \|\mu_i\|$. 
\begin{assumption}\label{nice_covariances}
There exists $\underline{\lambda} \leq 1 \leq \overline{\lambda}$ such that for all $t \geq 0$, $\underline{\lambda} \leq\sigmamin(\Sigma_i) \le \sigmamax(\Sigma_i) \leq \overline{\lambda}$. Note that the same bound immediately holds for $\sigmamin(t), \sigmamax(t)$ as a result. 
\end{assumption}
We can prove analogous bounds to Lemmas~\ref{ratio_inequality} and~\ref{diff_score_expo} in terms of these parameters, see Lemmas~\ref{ratioboundgaussian} and~\ref{diff_score_expo} in Appendices~\ref{app:nablaratiolemma} and~\ref{app:ratioboundgaussian}. Using these ingredients, we prove in Appendix~\ref{app:gauss_thm} the following:
\begin{restatable}{theorem}{mastersgaussiantheorem}
\label{thm::well_conditioned_gaussian_theorem}
Take any $\Sinit \subset \Send\subset[K]$. For sufficiently small $\epsilon$, there exists $T_{\rm lower}$ and $T_{\rm upper}$ such that
$T_{\mathrm{lower}} \le \frac{1}{2}\ln \left(2d\frac{\overline{\lambda}-\underline{\lambda}}{\underline{\lambda}} +\frac{1}{\underline{\lambda}}w(S_{\rm init},\Send)^2\right)+\ln \frac{1}{\epsilon}$ and also $T_{\mathrm{upper}} \ge \ln \Delta(\Send)+\frac{1}{2} \ln \underline{\lambda} -\ln 4-\frac{1}{2}\ln \ln \left(\frac{\overline{\lambda}\sqrt{K\overline{W}}\left[\left(\overline{\lambda}-\underline{\lambda}\right)^2(\overline{R}(0)^2+\overline{\lambda}d)+\overline{R}(0)^2\right]}{\underline{\lambda}^2\Delta(\Send)^2 \epsilon^2} \right)$ and such that for any $\widehat{T} \in (T_{\mathrm{lower}},T_{\mathrm{upper}})$, $\TV(\modrevlaw{\Sinit}{\wh{T}}{},p^{\Send}) \lesssim \epsilon$.
\end{restatable}

\noindent To get intuition for the bound, consider the simpler scenario where the covariances are the identity matrix.
\begin{example}\label{ex:identity_gaussians} ($K$ Gaussians with identity covariance) Let $\Sigma^i_0=\mathrm{Id}$ for all $i \in [K]$. Then, for any $S_{\mathrm{init}}\subset  \Send \subset [K]$, $T_{\mathrm{lower}} = \ln w(S_{\mathrm{init}}, \Send) +\ln 1/\epsilon$ and $T_{\mathrm{upper}} = \ln \Delta(\Send) -\ln 4-\frac{1}{2}\ln \ln \frac{\overline{R}(0)^2\sqrt{K\overline{W}}}{ \epsilon^2\Delta(\Send)^2 }$.
The dominant terms are $\ln w(S_{\mathrm{init}}, \Send)$ and $\ln \Delta(\Send)$, which depend on the intra- and inter-group distances of the means. In Fig.~\ref{fig:gaussian_hierarchy}, we plot these critical times and the final membership of the noised then denoised points for a Gaussian mixture. We see that our bounds match real class membership. 
\end{example}
\begin{figure}[!h]
    \centering
    \includegraphics[width=0.7\linewidth]{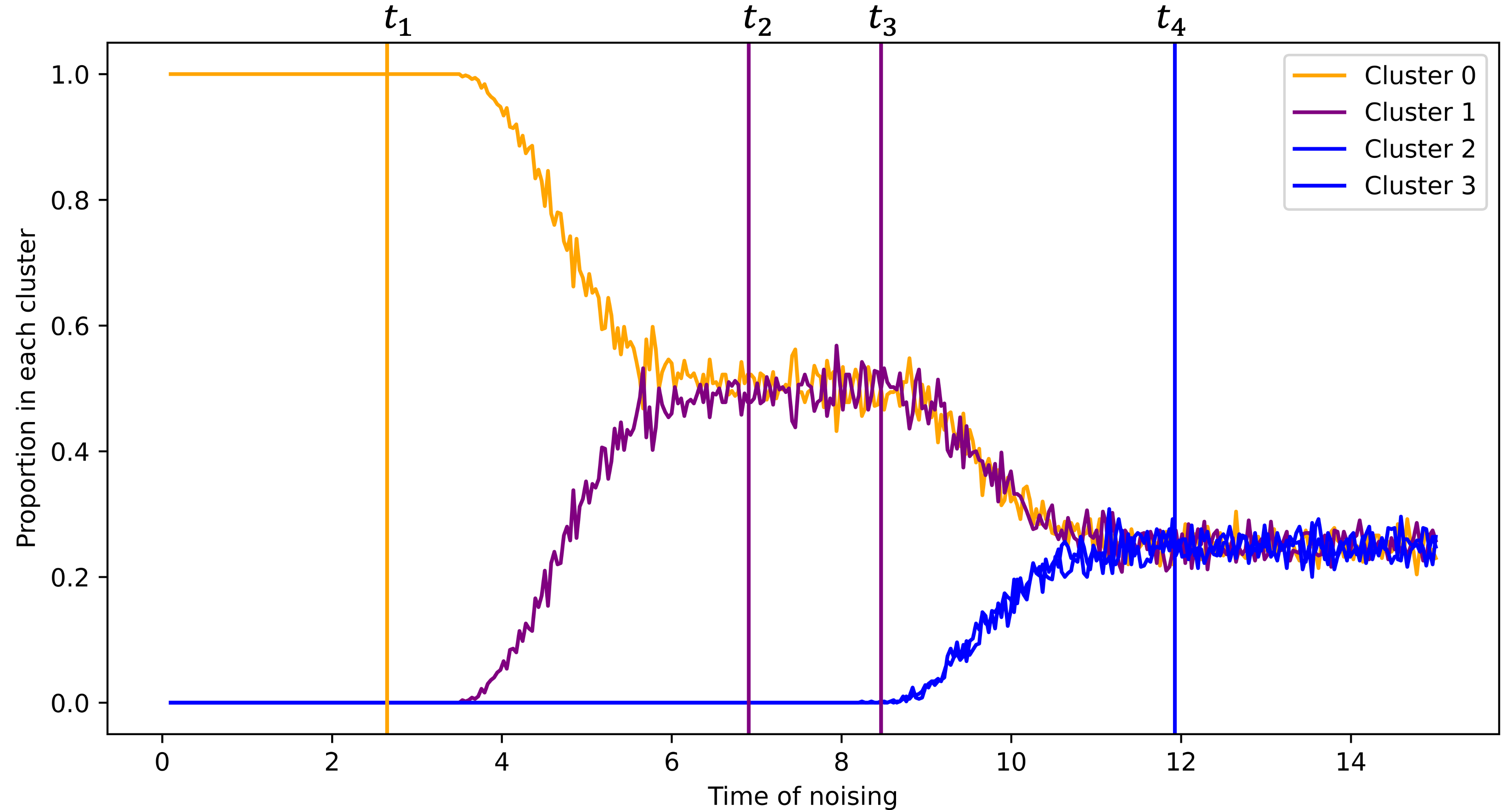}
    \caption{Example of critical times and proportion in each cluster as a function of noise timesteps. A point belongs to a cluster if its distance to the cluster mean is $\leq 5$. All clusters have identity covariance; cluster $0$ has mean $(-15100)$; cluster $1$ has mean $(-14900)$; cluster $2$ has mean $(14900)$; cluster $3$ has mean $(15100)$. We compute the thresholds $t_1,t_2,t_3,t_4$ with $\epsilon=0.1$ with the formulae from Example~\ref{ex:identity_gaussians}. By noising for $t \leq t_1$, we only sample from cluster $1$. By noising for $t \in [t_2,t_3]$, we sample from clusters $0,1$. By noising for $t \geq t_4$, we sample from all clusters. \vspace{-1em}}
    \label{fig:gaussian_hierarchy}
\end{figure}

\section{Hierarchy of classes}
\label{sec:hierarchy}
In this section, we consider a \emph{sequence} of critical windows that enable sampling from a sequence of nested sub-mixtures. Figure~\ref{fig:gaussian_hierarchy} hints at this idea, that as we noise for longer time periods, we sample from more and more components. Before we continue, it will be useful to formalize our model of a hierarchy of classes as a tree.  
\begin{definition}\label{def:tree}
We define a \textbf{mixture tree} as a tuple $(T,h,f,R)$. A tree $T=(V,E)$ of height $H= O(\sqrt{\ln R})$ is associated with a height function $h:V \to \mathbb{N}$ mapping vertices to their distance to the root and a function $f:V\to 2^{[K]}\backslash\{\emptyset\}$ satisfying the following: (1) $f(\textrm{root})=[K]$; (2) if $u$ is a parent of $v$, $f(v) \subset f(u)$; (3) for distinct $i,j \in [K]$ with leaf nodes $w,v$ such that $i \in f(w), j \in f(v)$, if $u$ is the lowest common ancestor of $w,v$, then $\|\mu_i-\mu_j\| \in (1\pm\delta)\ln \frac{R}{2^{h(u)^2}}$ with $\delta<0.01$. 
\end{definition}
Intuitively, the sequence of increasing critical windows of the noising and denoising process acts as a path up a mixture tree from some leaf. Within each critical window, the noising and denoising process is sampling from every class in the corresponding node in the path to the root. The class means have to be within a constant factor of $\ln \frac{R}{2^{h^2}}$, where $h$ is the height of their lowest common ancestor, to both ensure statistical separation from components outside the target mixture and small statistical distance within the target mixture. To make the critical times more explicit, we consider the setting of a mixture of identity covariance Gaussians (see proof in Appendix~\ref{app:hierarchy}): 
\begin{restatable}{theorem}{thmhierarchy}\label{thm:hierarchy_example}
    Let all $\Sigma^i = \mathrm{Id}$, $\|\mu_i\|=R$, and $w_i=\frac{1}{K}$. For $i \in [K]$, consider the path $u_1,u_2,u_3,\dots,u_{H'}$ where $u_1$ is the leaf node with $i \in f(u_1)$ and $u_{H'}$ is the root. There exists $k \in [1,2,\dots,H']$, sufficiently large $R,H'$, and sufficiently small $\epsilon$ such that there is a sequence of times $T_1<T_{2}<\dots<T_{k}$ with $\TV(\modrevlaw{\{i\}}{T_\ell}{}, p^{f(u_\ell)}) \lesssim \epsilon$.
\end{restatable}
This model also captures the intuition that diffusion models select more substantial features of an image before resolving finer details. When one ascends a tree of sub-mixtures from a leaf to the root through noising, one is essentially adding contributions to the score from more and more components of the mixture. Similarly, when a diffusion model samples from a hierarchy, it can be seen as ignoring negligible components of the mixture from the score until it reaches the end component. 

\section{Critical windows in Stable Diffusion}
\label{sec:sd}
In this section, we give an example of a critical window in Stable Diffusion v2.1 (SD2.1) to corroborate our theory. We generated images of cars and chose color, background, and size as our features. We noised and denoised each image for $t=350$ to $490$ time and plotted percentage of feature agreement with the base image vs. time (Figure~\ref{fig:critical_time_car}). We produced $250$ images from SD2.1., using $500$ time steps from the DDPM scheduler \citep{NEURIPS2020_4c5bcfec} and the prompt "\texttt{Color splash wide photo of a car in the middle of empty street, detailed, highly realistic, brightly colored car, black and white background}." (see Figure~\ref{fig:hand-labeled1}). We used the CLIP with the ViT-B/32 Transformer architecture to label our images \citep{radford2021learning} according to the subject matter of their background ("\texttt{car in a city/on a road/in a field}"), color intensity ("\texttt{black or white/pale colored/brightly colored car}"), and size ("\texttt{big/medium/small car}"). We used the prompt with the largest dot product with the image according to CLIP as the feature label. Note the background feature: from time step 480 to 490, the percentage of images with the same background as the original image drops by 25\%. The size feature also sees a substantial drop from 470 to 490 by 15\%. The agreement for the color also decreases significantly but the drop is much less sharp and occurs between time steps 450 to 470. Our theory for hierarchical sampling suggests that the diffusion model selects the car's size and background before deciding the color. This experiment also implies that critical windows can exist when the target mixture is different from the initial mixture, because noising and denoising to $[450,470]$ timesteps is sampling from the superclass of cars that have the same size and background as the original car but different colors. 

\begin{figure}[!h]
    \centering
    \includegraphics[width=\linewidth]{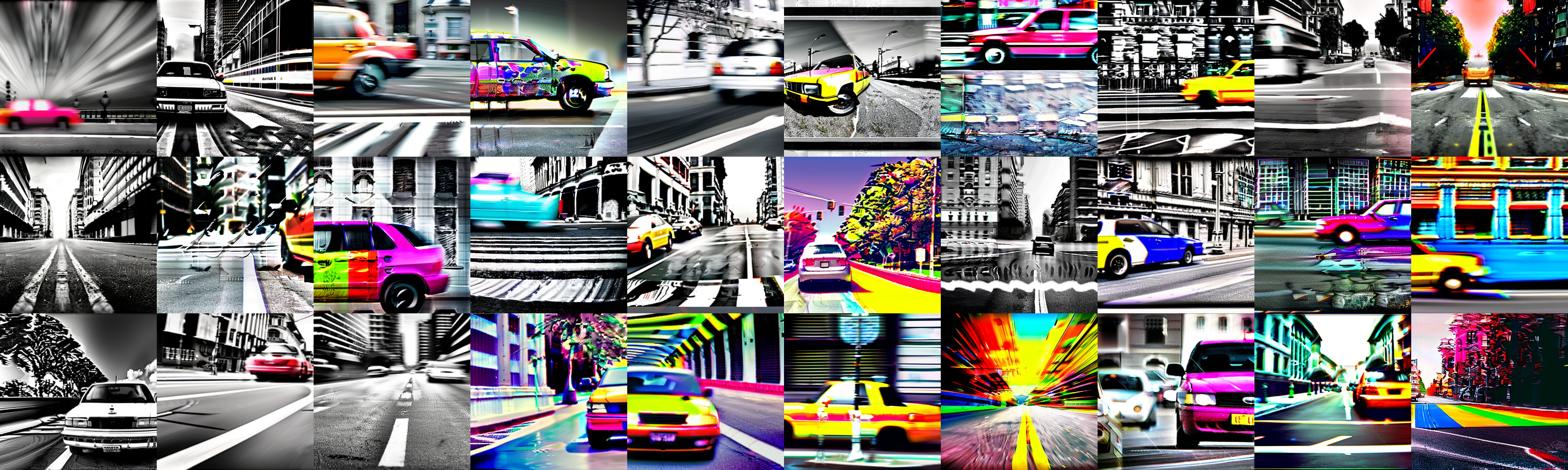}
    \caption{Example images of cars generated by SD2.1 that we subsequently noised and denoised to produce Figure~\ref{fig:critical_time_car}.}
    \label{fig:hand-labeled1}
\end{figure}

\begin{figure}[!h]
    \centering
    \includegraphics[width=0.8\linewidth]{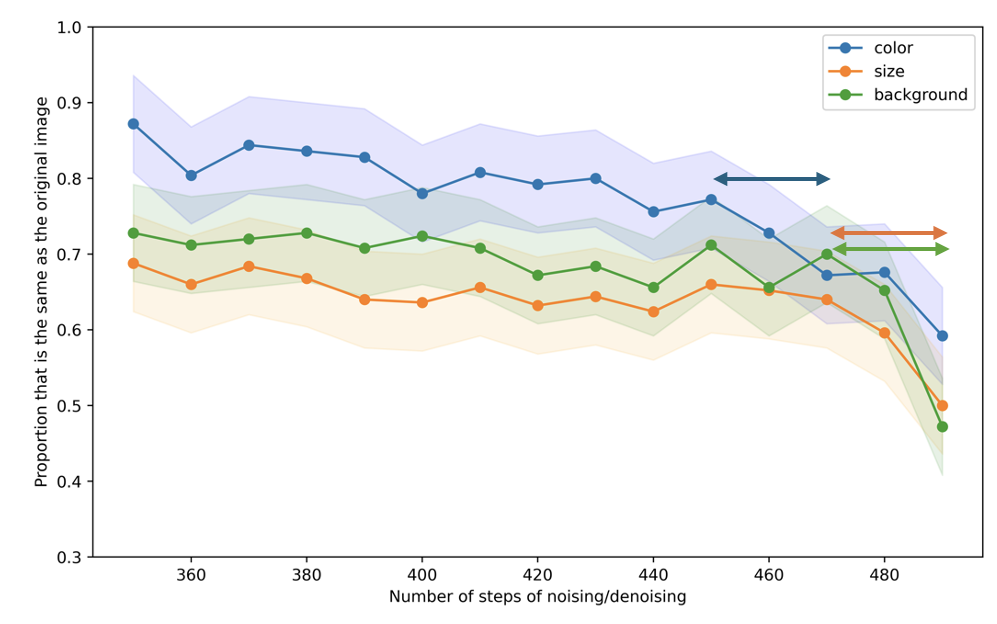}
    \caption{Percentage of agreement vs. noising amount in the experiment on images of cars generated by SD2.1 (see Section~\ref{sec:sd} for details). The critical window for each feature is demarcated with double-sided horizontal arrows.}
    \label{fig:critical_time_car}
\end{figure}

\section{Applications to fairness and privacy}

\subsection{Fairness} 
\label{sec:fair}
Generative models can reproduce social biases with their outputs \citep{luccioni2023stable}. Here we ask whether potentially biased features like gender have critical windows, as this could help design specific interventions to apply to diffusion model within that narrow range to improve image diversity \citep{raya2023spontaneous}. We studied outputs of photo portraits of laboratory technician on SD2.1 \citep{luccioni2023stable}, sampled $200$ images (see Figure~\ref{fig:hand-labeled2} for examples), and created an analogous plot of critical times (Figure~\ref{fig:critical_fair}). To determine gender, we against used a CLIP model and tested whether a given image had higher dot product with the prompt appended with "\texttt{, male}" or "\texttt{, female}". We can see a large drop in agreement between $t=80$ and $t=84$, from over 80\% to roughly 50\%, suggesting a critical window for the gender feature.  If the male and female classes are not well-separated at time $t=80$, then the noising and denoising procedure should result in a more equal mix of images from both classes. This confirms the intuition of our theory that different categories are well-separated before a critical window. 
\begin{figure}[!h]
    \centering
    \includegraphics[width=\linewidth]{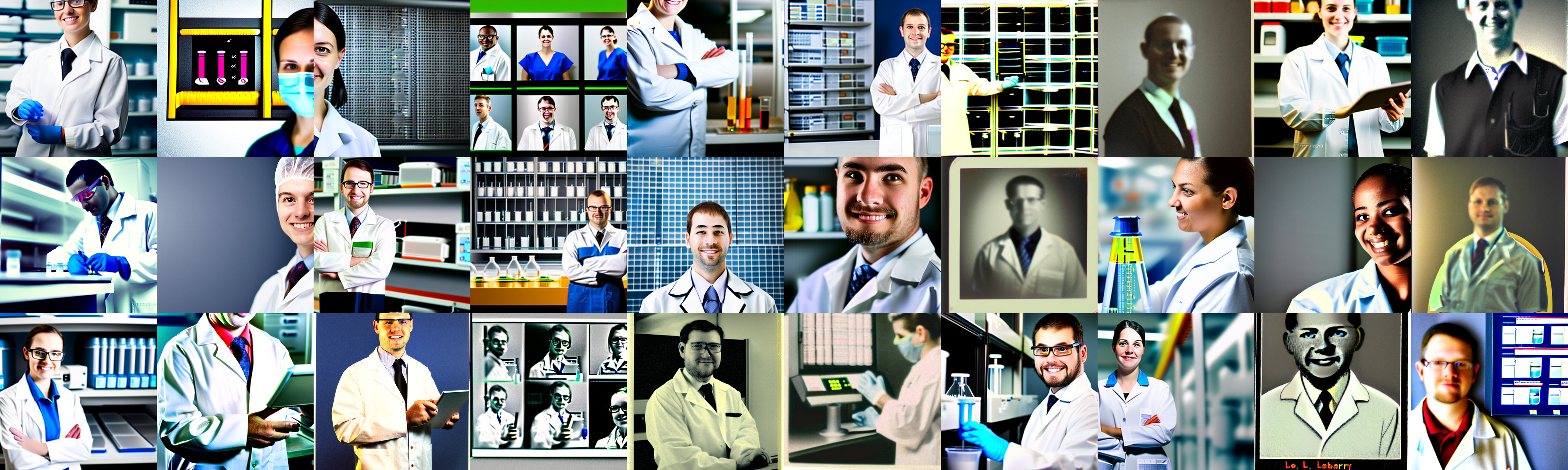}
    \caption{Example images generated by SD2.1 from the prompt ``\texttt{Photo portrait of a laboratory technician},'' that we subsequently noised and denoised for $100$ timesteps to produce Figure~\ref{fig:critical_fair}.}
    \label{fig:hand-labeled2}
\end{figure}
\begin{figure}[!h]
    \centering
    \includegraphics[width=0.8\linewidth]{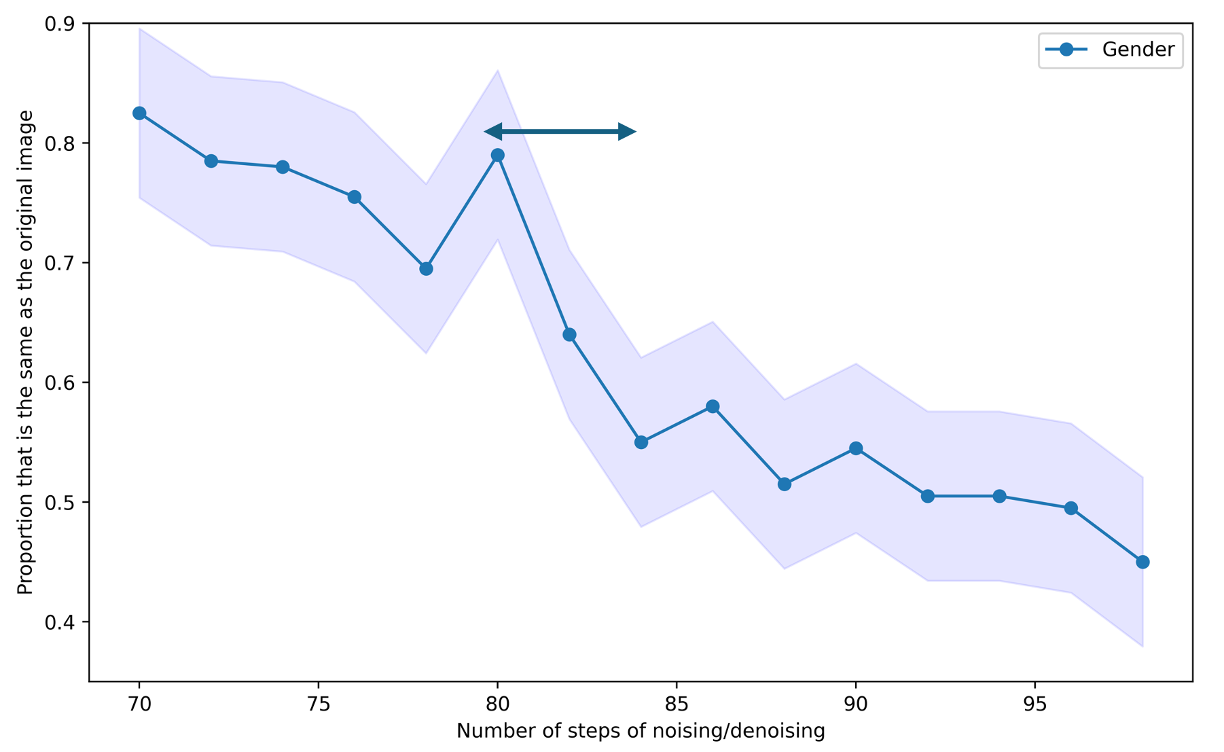}
    \caption{Critical window for gender feature in the experiment on images of laboratory technicians generated by SD2.1. Critical window demarcated with double-sided horizontal arrows. }
    \label{fig:critical_fair}
\end{figure}\\

\subsection{A new membership inference attack}
Membership Inference Attacks (MIAs) are a class of privacy attacks that try to identify whether a candidate sample belonged to training data \citep{DBLP:conf/sp/ShokriSSS17}, and are relevant for diffusion models because of their substantial privacy and copyright risks \citep{10.5555/3620237.3620531}. We present a simple MIA (NoiseDenoise) based on the distance between a candidate and noised and denoised copies. Let $\Theta$ be the set of possible models and $\mathcal{X}$ be the set of possible inputs, herein the diffusion model and candidate image, respectively. Let $\mathcal{D}_{\rm train}$ be the training data and $\mathcal{D}$ be the distribution from which the training data was drawn. To evaluate a MIA, we sample with probability $\frac{1}{2}$ some $x \sim \mathcal{D}_{\rm train}$ and otherwise sample $x \sim \mathcal{D}$. We rigorously describe our attack $\textrm{NoiseDenoise}(\mathcal{M}):\Theta \times \mathcal{X}\to \mathbb{R}$. For a diffusion model $\theta \in \Theta$, let $q(x_t|x_{t-1})$ denote the $T$-step forward process and $p_\theta(x_{t-1}|x_{t})$ denote the learned denoising process. Let $\underline{T} \in (0,T)$ denote the number of noising steps of our attack and $N$ the number of samples of our attack. For $x \in \mathcal{X}$, we generate $N$ samples $\tilde{x}^i_{\underline{T}} \sim q(x_{\underline{T}}|x_0=x)$ and $\tilde{x}^i_0 \sim p_\theta(x_0|x_T=\tilde{x}^i_{\underline{T}} )$ for $i \in [N]$. Our attack is the average L2 difference between $\tilde{x}^i_0$ and $x$ for $i \in [N]$, and we predict $x$ to belong to the training data if $\mathcal{M}(\theta,x) \le \tau$,
\begin{align}
\mathcal{M}(\theta,x)  = \frac{1}{N} \sum_{i \in [N]} \|\tilde{x}^i_0-x\|_2.
\end{align}

Note that this method has already demonstrated some promising results in identifying whether an image was generated by a diffusion model \citep{finalproj}. We present a conceptual explanation of our attack as follows. A diffusion model $\theta$ implicitly defines a pushforward distribution $\theta_*\gamma^d$ on images. For a candidate image $x$, we can view $\theta_*\gamma^d$ as a mixture of a ball around $x$, i.e. some $B_{R}(x)$ with $R >0$, and the remainder of the distribution. Within a ball $B_{R}(x)$, we expect diffusion models to typically place more of the mass close to $x$ when $x \in \mathcal{D}_{\rm train}$ because training data have smaller losses. Thus we have greater separation from the remainder of the distribution for training data, and based on our theoretical framework, we can noise and denoise $x \in \mathcal{D}_{\rm train}$ for more time steps than $x \notin \mathcal{D}_{\rm train}$ and obtain samples close to $x$. Our justification is similar to the logic characterizing diffusion model memorization in the independent and concurrent work of \citet{biroli2024dynamical}. \cite{biroli2024dynamical} considers the volume of neighborhoods around training data to identify critical times in their "collapse" regime, while we relate the size of these neighborhoods to our critical window theorems and develop these intuitions into a MIA. Additionally, this technique can be viewed as the diffusion model analogue of language model methods which perturb the inputs as part of MIAs \citep{li-etal-2023-mope} or machine-generated text detection \citep{10.5555/3618408.3619446}. 

Our attack was tested on a DDPM that was trained on CIFAR-10 in \cite{10.5555/3618408.3618757} and we compare it to their methods $\mathrm{SecMI}_{\textrm{stat}}$ and $\mathrm{SecMI}_{\textrm{nn}}$. Both of their attacks exploit a deterministic approximation of the forward and reverse process of a DDPM to estimate the sampling error of a candidate image. $\mathrm{SecMI}_{\textrm{stat}}$ is the error itself while $\mathrm{SecMI}_{\textrm{nn}}$ is a neural network trained on the errors at different timesteps. We set $N=10$ and $\wh{T}=50$ (with $T=100$), and compare all methods with 1000 training data samples and 1000 held-out samples. As in \cite{10.5555/3618408.3618757}, we present ROC curves, AUC statistics, and TPRs at low FPRs of all MIAs, see Figure~\ref{fig:roc_curve} and Table~\ref{tbl:roc}. Both Figure \ref{fig:roc_curve} and Table \ref{tbl:roc} show that $\mathrm{SecMI}_{\textrm{stat}}$ and $\mathrm{SecMI}_{\textrm{nn}}$ outperform NoiseDenoise. However, $11$ of $23$ of the train points NoiseDenoise identifies at $\text{FPR}=0.01$ and $21$ of $140$ of the train points identified at $\text{FPR}=0.05$ are not classified correctly by $\mathrm{SecMI}_{\textrm{stat}}$ or $\mathrm{SecMI}_{\textrm{nn}}$ at the same FPR thresholds, suggesting NoiseDenoise can serve as a \emph{complementary approach} to these methods.
\begin{table}[h]
\centering
\begin{tabular}{|c|c|c|c|c|}
\hline
     Method & AUC & $\text{TPR}_{.01}$ & $\text{TPR}_{.05}$ \\
     \hline
  $\textrm{NoiseDenoise}$ &  $.6636
$   & $.023 $        &  $.14 $ \\  
  $\mathrm{SecMI}_{\textrm{stat}}$ &  $.8847$   & $.073 $        &  $.344 $ \\  
    $\mathrm{SecMI}_{\textrm{nn}}$ &  $.9132$   & $.245$        &  $.609 $ \\  
\hline
\end{tabular}

\caption{For each attack, we report the AUC, TPR at FPR $.01$, and TPR at FPR $.05$.}
\label{tbl:roc}
\end{table}
\begin{figure}[!h]
    \centering
    \includegraphics[width=\linewidth]{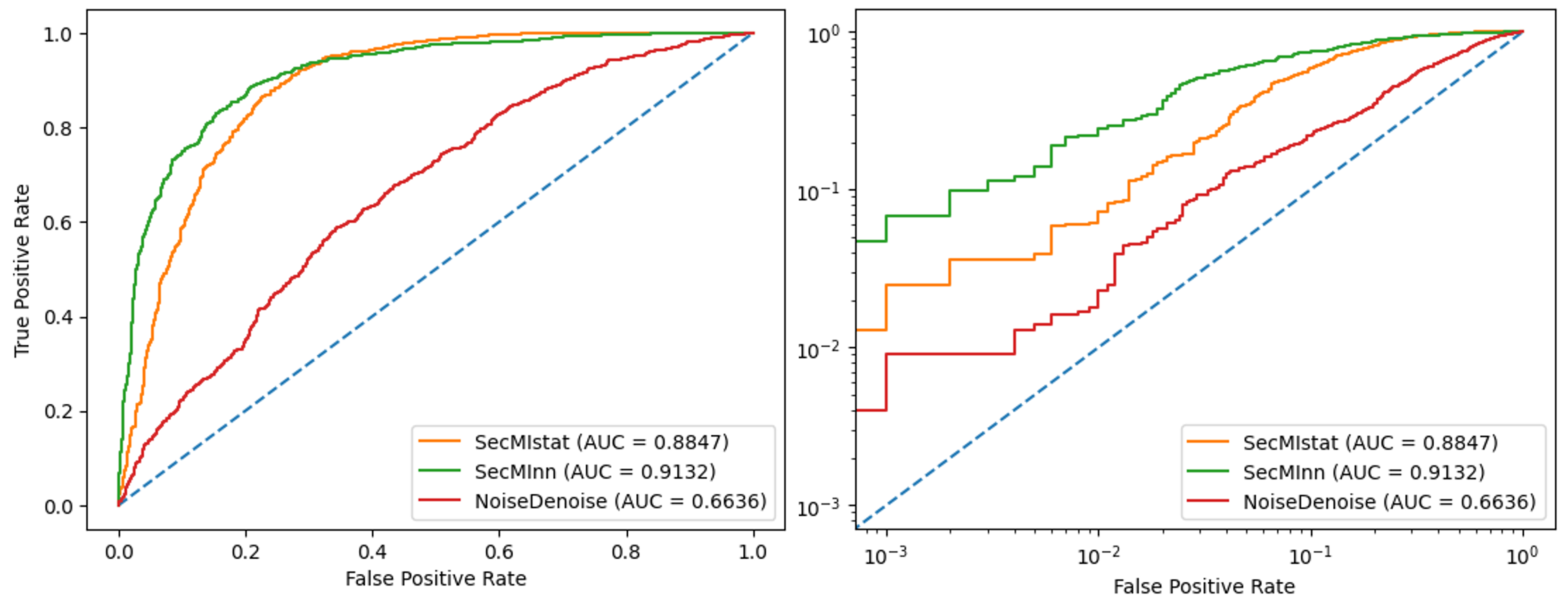}
    \caption{ROC curves of different methods.}
    \label{fig:roc_curve}
\end{figure}

\section{Conclusion}
In this paper, we considered noising and denoising samples from a mixture model and the resulting critical times of this process over which features emerge. We provide theory for the empirical observation from~\citet{raya2023spontaneous},~\citet{biroli2024dynamical},~\citet{georgiev2023journey}, and~\cite{sclocchi2024phase} that discrete features are decided within short windows in the sampling process. This same question was studied mathematically in recent and concurrent works~\citep{raya2023spontaneous,sclocchi2024phase,biroli2024dynamical}, and our rigorous non-asymptotic bounds for hierarchical mixture models nicely complement the precise statistical physics-based insights derived in those works.

We also present preliminary experiments describing critical windows for features in SD2.1, and demonstrate our framework's value for fairness and privacy. Our main contribution was identifying and proving a relationship between sampling from a target sub-mixture and statistical distances between components in the initial and target sub-mixture and remaining components. 

\paragraph{Limitations and future directions.} The most immediate technical follow-up would be to eliminate the logarithmic dependence on dimension for $T_{\rm upper}$ for more general distributions than just mixtures of well-conditioned Gaussians. Intuitively, the total variation of two distributions in high-dimensions should be close to $1$ unless the means are much closer than $\Theta(\sqrt{d})$ from each other. Future research that lower bounds the total variation between isotropic log-concave distributions $\Omega(\sqrt{d})$ distance apart would improve the bounds we present in our paper.

On the conceptual front, another direction is to discover analogues of critical windows for continuous features. For example, when we noise and denoise a picture of an orange car, we should expect that it takes fewer time steps to see pictures of red or yellow cars than purple cars, because orange is more similar to red or yellow than purple. Such features, e.g., color, height, and orientation, more naturally belong to a \emph{continuum} rather than discrete bins, but our critical window theorems require strong statistical separation between components inside and outside the target sub-mixture and cannot capture this phenomenon. As more features in natural data are continuous than discrete, extensions of our critical window theorems to continuous features could expand the usefulness of this work.

This work presents exciting opportunities for empirical research in designing samplers and interpretability. \cite{raya2023spontaneous} experimented with initializing the sampling procedure right before the point of symmetry breaking, in order to improve efficiency and diversity of sampling. Our critical window theorems and experiments indeed suggest that by concentrating on timesteps within the critical windows in which new features are defined rather than time periods when only existing features are refined, we may be able to improve generation quality. Another interesting empirical direction would be to describe features and hierarchies of features in realistic conditional diffusion models. This would require a methodology to systematically identify and extract features from image generations as well as accurately determine if given images contain a feature. With respect to the latter, it seems impractical to label enough data to follow the supervised classifier approach of \citep{raya2023spontaneous,biroli2024dynamical}, while the CLIP-based approach of our work and~\citet{georgiev2023journey} can label images with novel prompts in a zero- or few-shot fashion.  \cite{sclocchi2024phase}'s definition of a feature as the embeddings in an image classifier is an alternative but \textit{a priori} may not be useful for practitioners unless the embeddings translate into realizable concepts. The former challenge would be identifying easily-understandable features in highly diverse and complex image generations. Strong pretrained multi-modal models may again provide a solution to this problem: \cite{gandelsman2024interpreting} automatically generates image descriptions with GPT3.5 to identify salient features for their goal to understand the role of attention heads in CLIP. Machine-driven hypothesis generation of important features in image generations combined with multimodal labelers could enable a complete and structured understanding of feature emergence in diffusion models.

\paragraph{Impact statement}
This paper is largely theoretical in nature but also describes a new membership inference attack against diffusion models. This could potentially impact the privacy of their training data. We hope this paper spurs further research into the fundamental mechanisms behind memorization in diffusion model so that these risks may be limited in the future. Furthermore, the diffusion model that we tested our MIA on was only trained on CIFAR-10 data, limiting the privacy risks of our study.

\section*{Acknowledgments}

SC would like to thank Aravind Gollakota, Adam Klivans, Vasilis Kontonis, Yuanzhi Li, and Kulin Shah for insightful discussions on diffusion models and Gaussian mixtures. ML would like to thank Andrew Campbell and Jason Wang for thoughtful conversations about diffusion models and the applications of this work. 

\bibliography{refs}
\bibliographystyle{icml2024}

\newpage
\appendix
\onecolumn
\section{Deferred proofs}

\subsection{Proof of Lemma~\ref{ratio_inequality}}\label{app:ratio_inequality}
\ratioinequalitylemma*
\begin{proof}
We exhibit the leftmost equality by noting $dPdQ = \frac{1}{4}\left((dP+dQ)^2-(dP-dQ)^2\right)$,
\begin{align}
\mathbb{E}_{x \sim P} \left[\frac{dQ}{dP+dQ} \right] &= \int \frac{dPdQ}{d(P+Q)}\\
&= \frac{1}{4}\left[\int \frac{(dP+dQ)^2-(dP-dQ)^2}{d(P+Q)}\right]\\
&= \frac{1}{4}\left[2-\int \frac{(dP-dQ)^2}{d(P+Q)}\right]\\
&= \frac{1}{2}\left[1-\LC(P,Q)\right].
\end{align}
The first inequality follows from $\LC(P,Q) \geq \frac{1}{2}\Hellinger^2(P,Q)$ (see p.48 in \cite{GVK024181773}). The second inequality follows from rearranging $4\TV^2(P,Q) \leq \Hellinger^2(P,Q)(4-\Hellinger^2(P,Q))$ (see p.47 in \cite{GVK024181773}) into $1-\frac{1}{2}\Hellinger^2(P,Q) \leq \sqrt{1-\TV^2(P,Q)}$. 
\end{proof}

\subsection{Proof of Lemma~\ref{diff_score_expo}}\label{app:diff_score_expo}
\begin{lemma}\label{bound_score_slc}
Under Assumption \ref{slc_components}, the Hessian of $\ln p^i_t$ for $i \in [K]$ is between
\begin{align}
  \frac{1}{e^{-2t}\Psi^2 + 1-e^{-2t}}\mathrm{Id}\preceq \nabla^2 (-\ln p^i_t) \preceq \frac{1}{1-e^{2t}} \mathrm{Id}.
\end{align}
\end{lemma}
\begin{proof}
Using the preservation of strong log-concavity (see p.71 in in \cite{10.1214/14-SS107} or \cite{0d881dbfc2b94408bf4ca044a554766f}), we find that for $i \in [K]$,
\begin{align*}
  p^i_t &\in \SLC(e^{-2t}\Psi^2+(1-e^{-2t}),d)  .
\end{align*}
By Proposition 2.23 of \cite{10.1214/14-SS107}, this implies $\nabla^2 (-\ln p^i_t) \succeq \frac{1}{e^{-2t}\Psi^2+(1-e^{-2t})}.$ For the second inequality, we follow the proof of Proposition 7.1. in \cite{10.1214/14-SS107} for the convolution $X^i_t = e^{-t}X^i_0+\mathcal{N}(0,(1-e^{-2t})\mathrm{Id})$. Let $X:=e^{-t}X^i_0,Y:=\mathcal{N}(0,(1-e^{-2t})\mathrm{Id}),Z:=X^i_t$, and let $p_X,p_Y,p_Z$ be their respective densities. Because
\begin{align}
\nabla (- \ln p_Z)(z) =\frac{-\nabla p_Z(z)}{p_Z(z)}= \mathbb{E}_{X \sim p_X}[p_Y(z-X)\cdot \nabla \left(- \ln p_Y(z-X)\right)] \cdot \frac{1}{p_Z(z)}=\mathbb{E}[\nabla(- \ln p_Y)(Y)|X+Y=z], 
\end{align}
we can compute the Hessian with the product rule,
\begin{align}
\nabla^2(- \ln p_Z)(z)&=\nabla \left\{\mathbb{E}_{X \sim p_X}[p_Y(z-X)\cdot \nabla \left(- \ln p_Y(z-X)\right)] \cdot \frac{1}{p_Z(z)}\right\}\\
&= -\mathbb{E}_{X \sim p_X}\left[p_Y(z-X)\nabla \ln p_Y(z-X) (\nabla \ln p_Y(z-X))^\top\right] \cdot \frac{1}{p_Z(z)}\\
&+\mathbb{E}_{X \sim p_X}[p_Y(z-X)\nabla^2(- \ln p_Y(z-X))]\cdot \frac{1}{p_Z(z)}\\
&+\mathbb{E}_{X \sim p_X}[p_Y(z-X)\nabla(\ln p_Y(z-X)) ] \cdot \frac{1}{p_Z(z)}\cdot \frac{\nabla p_Z(z)}{p_Z(z)} \\
&= -\mathbb{E}[\nabla \ln p_Y(Y) (\nabla \ln p_Y(Y))^\top|X+Y=z]+\mathbb{E}[\nabla^2(- \ln p_Y(Y))|X+Y=z]\\
&+\left(\mathbb{E}[\nabla \ln p_Y(Y)|X+Y=z]\right)^{\otimes 2}\\
&=-\mathrm{Var}(\nabla (-\ln p_Y(Y))|X+Y=z) + \mathbb{E}[\nabla^2(-\ln p_Y(Y))|X+Y=z] \\
&\preceq \frac{1}{1-e^{-2t}}\mathrm{Id},
\end{align}
where  the last line uses $\mathrm{Var}(\nabla (-\ln p_Y(Y))|X+Y=z) \succeq 0$ and $\mathbb{E}[\nabla^2(-\ln p_Y(Y))|X+Y=z] =\frac{1}{1-e^{-2t}}\mathrm{Id}$.
\end{proof}

\begin{lemma}\label{lemm:score_at_origin}
For $t > 0.001$, we have the following inequality on the score at the origin,
\begin{align}
\|\nabla \ln p^i_t(0)\| \lesssim  e^{-t}\left[\|\mu_i\|+M\right].
\end{align}
\end{lemma}
\begin{proof}
By the definition of a convolution, we can explicitly compute
\begin{align}
\nabla \ln p^i_t(0) &= \frac{-\int_{\mathbb{R}^d}p^i_0\left(u\right) f_{\mathcal{N}(0,(1-e^{-2t})\mathrm{Id})}(-ue^{-t})\frac{0-ue^{-t}}{1-e^{-2t}} du}{\int_{\mathbb{R}^d}p^i_0\left(u\right) f_{\mathcal{N}(0,(1-e^{-2t})\mathrm{Id})}(-ue^{-t})du} = \frac{e^{-t}}{1-e^{-2t}} \frac{\int_{\mathbb{R}^d}p^i_0\left(u\right) f_{\mathcal{N}(0,(1-e^{-2t})\mathrm{Id})}(-ue^{-t})udu}{\int_{\mathbb{R}^d}p^i_0\left(u\right) f_{\mathcal{N}(0,(1-e^{-2t})\mathrm{Id})}(-ue^{-t})du}.
\end{align}
Note that for all $t \geq 0.001$, $f_{\mathcal{N}(0,(1-e^{-2t})\mathrm{Id})}$ is $\Omega$-Lipschitz for some $\Omega > 0$. Thus, we can bound the distance between the numerator and $\mu_if_{\mathcal{N}(0,(1-e^{-2t})\mathrm{Id})}(0)$ with the triangle inequality and Assumption \ref{higher_moments}, 
\begin{align}
&\left\|\int_{\mathbb{R}^d}p^i_0\left(u\right) f_{\mathcal{N}(0,(1-e^{-2t})\mathrm{Id})}(-ue^{-t})udu-\mu_if_{\mathcal{N}(0,(1-e^{-2t})\mathrm{Id})}(0)\right\|\\
&\leq \int_{\mathbb{R}^d}p^i_0\left(u\right) \left\|f_{\mathcal{N}(0,(1-e^{-2t})\mathrm{Id})}(-ue^{-t})-f_{\mathcal{N}(0,(1-e^{-2t})\mathrm{Id})}(0) \right\| \cdot \|u\|du\\
&\leq \Omega e^{-t}\int_{\mathbb{R}^d}p^i_0(u)\|u\|^2du \leq \Omega e^{-t} M.
\end{align}
The denominator also approaches $f_{\mathcal{N}(0,(1-e^{-2t})\mathrm{Id})}(0)$ at the rate of $O(e^{-t})$, and we can express a bound on the distance from $f_{\mathcal{N}(0,(1-e^{-2t})\mathrm{Id})}(0)$ in terms of $M$ using Jensen's inequality,
\begin{align}
\left\|\int_{\mathbb{R}^d}p^i_0\left(u\right) f_{\mathcal{N}(0,(1-e^{-2t})\mathrm{Id})}(-ue^{-t})du - f_{\mathcal{N}(0,(1-e^{-2t})\mathrm{Id})}(0)\right\| \leq e^{-t}\int_{\mathbb{R}^d}p^i_0\left(u\right) \|u\|du \leq e^{-t}M.
\end{align}
Thus there  exists $0 \leq \epsilon_1,\epsilon_2 \leq \max(\Omega,1)$ and $w \in \mathbb{S}^{d-1}$ such that for all $t \geq 0.001$, we have the score bound
\begin{align}
\left\|\nabla \ln p^i_t(0)\right\| &= \frac{e^{-t}}{1-e^{-2t}} \left\|\frac{f_{\mathcal{N}(0,(1-e^{-2t})\mathrm{Id})}(0)\mu_i+\Omega e^{-t}M\epsilon_1 w}{f_{\mathcal{N}(0,(1-e^{-2t})\mathrm{Id})}(0)+e^{-t}M\epsilon_2}\right\| \lesssim e^{-t}\left[\|\mu_i\|+M\right]. 
\end{align}
\end{proof}

\fourthscoredifflemma*
\begin{proof}
For $t < 0.001$, we can prove the lemma by directly appealing to the bounded fourth moments of the scores $\nabla \ln q_t(X), \nabla \ln p_t(X)$ by Assumption~\ref{score_bound_assumption}, 
\begin{align}
\mathbb{E}_{X \sim p^i_t} \|\nabla \ln p^j_t(X)-\nabla \ln p^\ell_t(X)\|^4 &\lesssim \mathbb{E}_{X \sim p^i} \left[\|\nabla \ln p^j(X)\|^4+\|\nabla \ln p^\ell(X)\|^4 \right]\lesssim \overline{M}.
\end{align}
For $t \geq 0.001$, it suffices to bound the difference with the scores of the standard normal by the triangle inequality,
\begin{align}
\mathbb{E}_{X \sim p^i_t} \|\nabla \ln p^j_t(X)-\nabla \ln p^\ell_t(X)\|^4 &\lesssim \mathbb{E}_{X \sim p^i_t} \|\nabla \ln p^j_t(X)-\nabla \ln f_{\mathcal{N}(0,\mathrm{Id})}(X)\|^4\\
&+ \mathbb{E}_{X \sim p^i_t} \|\nabla \ln p^\ell_t(X)-\nabla \ln f_{\mathcal{N}(0,\mathrm{Id})}(X)\|^4
\end{align}
Both terms with $p^j_t,p^\ell_t$ are controlled by the same procedure. For $j$, we can write 
\begin{align*}
\mathbb{E}_{X \sim p^i_t} \|\nabla \ln p^j_t(X)-\nabla \ln f_{\mathcal{N}(0,\mathrm{Id})}(X)\|^4 &\lesssim \mathbb{E}_{X \sim p^i_t} \|(\nabla (\ln p^j_t- \ln f_{\mathcal{N}(0,\mathrm{Id})}))(X)-\nabla (\ln p^j_t-\ln f_{\mathcal{N}(0,\mathrm{Id})})(0)\|^4 + \|\nabla \ln p^j_t(0)\|^4. 
\end{align*}
By Lemma~\ref{bound_score_slc},$\nabla^2(-\ln p^j_t+\ln f_{\mathcal{N}(0,\mathrm{Id})})$'s eigenvalues are in $[\frac{e^{-2t}(1-\Psi^2)}{e^{-2t}\Psi^2 + 1-e^{-2t}},\frac{e^{-2t}}{1-e^{2t}}] \subset [-\Psi^2 e^{-2t},1000e^{-2t}]$. Thus $\nabla \ln p^j_t- \ln f_{\mathcal{N}(0,\mathrm{Id})}$ is globally $1000 \Psi^2 e^{-2t}$-Lipschitz. Combining with Lemma~\ref{lemm:score_at_origin}, we can conclude 
\begin{align*}
\mathbb{E}_{X \sim p^i_t} \|\nabla \ln p^j_t(X)-\nabla \ln f_{\mathcal{N}(0,\mathrm{Id})}(X)\|^4 &\lesssim e^{-8t}\left[\mathbb{E}_{X \sim p^i_t} \Psi^8 \|X\|^4\right]+e^{-4t}\left[\|\mu_i\|^4+M^4\right]\\
&\lesssim e^{-4t}\left[\|\mu_j\|^4+M^4+M\Psi^8\right]
\end{align*}
\end{proof}

\subsection{Proof of Lemma~\ref{lem:scorediff_rewrite}}\label{app:scorediffrewrite}
\scorediffrewrite*
\begin{proof}
    This follows by some simple algebraic manipulations:
    \begin{align}
    &\bigl\|\nabla \ln p^{\Send}_t -\nabla \ln p_t\bigr\|^2= \left\|\frac{\sum_{i \in \Send} w_i \nabla p^i_t}{\sum_{i \in \Send} w_i p^i_t} - \frac{ \sum_{i\in[K]} w_i \nabla p^i_t}{ \sum_{i\in[K]} w_i p^i_t} \right\|^2\\
    &= \left\| \left(\frac{1}{\sum_{i \in \Send} w_i p^i_t}-\frac{1}{\sum_{i\in[K]} w_i p^i_t}\right)\sum_{i \in \Send} w_i \nabla p^i_t -\frac{\sum_{i \in [K]-\Send} w_i \nabla p^i_t}{\sum_{i \in [K]} w_i p^i_t} \right\|^2\\
    &=  \left\| \left(\frac{\sum_{i \in [K]-\Send} w_i p^i_t}{\left(\sum_{i \in \Send} w_i p^i_t\right)\left(\sum_{i\in[K]} w_i p^i_t\right)}\right)\sum_{i \in \Send} w_i \nabla p^i_t -\frac{\sum_{i \in [K]-\Send} w_i \nabla p^i_t}{\sum_{i \in [K]} w_i p^i_t} \right\|^2\\
    &=  \left( \frac{\sum_{i \in [K]-\Send} w_i p^i_t}{\sum_{i\in[K]} w_i p^i_t}\right)^2 \left(\left\|\frac{\sum_{i \in \Send} w_i \nabla p^i_t}{\sum_{i \in \Send} w_i p^i_t}-\frac{\sum_{i \in [K]-\Send} w_i \nabla p^i_t}{\sum_{i \in [K]-\Send} w_i p^i_t} \right\|^2\right)\\
    &=  \left( \frac{\sum_{i \in [K]-\Send} w_i p^i_t}{\sum_{i\in[K]} w_i p^i_t}\right)^2 \bigl\|\nabla \ln p_t^{\Send} - \nabla \ln p_t^{[K]-\Send} \bigr\|^2\,. \qedhere
    \end{align}
\end{proof}

\subsection{Proof of Lemma~\ref{corr:wassersteinlower}}\label{app:w2}
\begin{lemma}\label{mixture_tv_bound}
Consider mixture $P=\sum_i a_i P_i$ and mixture $Q=\sum_i b_i Q_i$. If $\TV(P_i,Q_j) \leq \epsilon$ for all $i,j$, then $$\TV(P,Q)\leq \epsilon.$$  
\end{lemma}
\begin{proof}
This is a simple application of triangle inequality,
\begin{align}
\frac{1}{2}\int |\sum_i a_i dP_i - \sum_j b_j dQ_j | \leq \frac{1}{2}\sum_i a_i\int |dP_i -  \sum_j b_j dQ_j |  \leq \frac{1}{2}\sum_i a_i \sum_j b_j \int |dP_i -   dQ_j |  \leq \epsilon.
\end{align}
\end{proof}

\begin{lemma}\label{kl_w2}
(Short-time regularization) Convolving with the normal distribution bounds $\KL$ in terms of $W_2$, 
\begin{align*}
\KL(p * \mathcal{N}(0,\sigma^2)||q * \mathcal{N}(0,\sigma^2)) \leq \frac{1}{2\sigma^2}W_2(p,q)^2
\end{align*}
\end{lemma}
\begin{proof}
By the joint convexity of $\KL$, it suffices to show for $p=\delta_x$ and $q=\delta_y$. Then,
\begin{align}
\KL(\mathcal{N}(x,\sigma^2)||\mathcal{N}(y,\sigma^2)) = \frac{\|x-y\|^2}{2\sigma^2}.
\end{align}
\end{proof}
\wassersteinlower*
\begin{proof}[Proof of bound on $T_{\mathrm{lower}}$ in Lemma~\ref{corr:wassersteinlower}]
Define $h^\ell_t$ to be the density of $e^{-t}X_t^\ell$ for $\ell \in [K]$. We apply Pinsker's inequality and treat the convolution with Gaussian noise in the forward process as a regularization parameter to control $\KL$ in terms of the Wasserstein-$2$ distance. For $i \in S_{\rm init}$ and $j \in \Send$ we can control the $\KL$ via Lemma~\ref{kl_w2},
\begin{align}
\TV(p^i_{T_{\rm lower}},p^j_{T_{\rm lower}}) \leq \sqrt{\KL(p^i_{T_{\rm lower}}||p^j_{T_{\rm lower}})} \leq \Wtwo(h^i_{T_{\rm lower}},h^j_{T_{\rm lower}}).
\end{align}
We use a coupling argument to control $\Wtwo(h^i_{T_{\rm lower}},h^j_{T_{\rm lower}})$. Let $\pi \in \Gamma(\overline{f}^i_0,\overline{f}^j_0)$ be the optimal coupling, and define the coupling in $\Gamma(p^i_{T_{\rm lower}},p^j_{T_{\rm lower}})$ that samples $(X,Y) \sim \pi$ and returns $(e^{-T_{\rm lower}}(X+\mu_i),e^{-T_{\rm lower}}(Y+\mu_j))$. The cost of this coupling is 
\begin{align}
\Wtwo(h^i_{T_{\rm lower}},h^j_{T_{\rm lower}}) \leq \sqrt{\mathbb{E}\|e^{{-T_{\rm lower}}}(X-Y)+e^{-{T_{\rm lower}}}(\mu_i-\mu_j)\|^2} &\leq e^{-{T_{\rm lower}}}\sqrt{2\left(\mathbb{E}\|X-Y\|^2+\|\mu_i-\mu_j\|^2\right)}\\
&\leq \sqrt{2}e^{-t} \left[\Upsilon + \|\mu_i-\mu_j\|\right]
\end{align}
Thus $\TV(p^i_t,p^j_t) \leq \sqrt{2}\left[ \|\mu_i-\mu_j\|+\Upsilon \right]e^{-t} \leq \epsilon$, and we can conclude by applying Lemma \ref{mixture_tv_bound} to obtain an overall bound on $\TV(p^{S_{\rm init}}_{T_{\rm lower}},p^{\Send}_{T_{\rm lower}})$.
\end{proof}

\begin{lemma}\label{lemma:convolve_subgaussian}
Consider sub-Gaussian random vectors $\{X_i\}_{i=1}^n$ in $\mathbb{R}^d$ with variance proxies $\{\sigma_i^2\}_{i=1}^n$. Let $S = \sum_{i=1}^n \alpha_i X_i$. Then, $S \in \subG_{d}(\sum_{i=1}^n \alpha_i^2 \sigma_i^2)$. 
\end{lemma}
\begin{proof}
This proof is trivial.
\end{proof}
\begin{lemma}\label{lemma:subgaussian_bound}
(Theorem 1.19 of \cite{highdimensional}) Let $X \in \subG_d(\sigma^2)$. Then, for any $t \geq 0$, 
\begin{align}
\mathbb{P}[\|X\| > t] \leq 6^d \exp(-t^2/(8\sigma^2)).
\end{align}
\end{lemma}
\begin{lemma}\label{lemma:subgaussian_to_tv}
Consider two random vectors $X,Y \in \mathbb{R}^d$ with probability density functions $P_X,P_Y$ and means $\mu_X,\mu_Y$ such that $X-\mu_X$ and $Y-\mu_Y$ are sub-Gaussian random vectors with variance proxy $\sigma^2$. Let $R=\sigma \sqrt{8d \ln 6+8 \ln 1/\epsilon}$. If $\|\mu_X-\mu_Y\|>2R$ then 
\[\TV(X,Y) \geq 1-\epsilon\]
\end{lemma}
\begin{proof}
By \ref{lemma:subgaussian_bound}, $\mathbb{P}(\|X-\mu_X\|\geq R),\mathbb{P}(\|Y-\mu_Y\|\geq R) \leq \epsilon,$ and $B_{\leq R}(\mu_X)$ and $B_{\leq R}(\mu_Y)$ are disjoint by definition. Thus, 
\begin{align}
\TV(X,Y) = \frac{1}{2}\int_{\mathbb{R}^d} |dP_X-dP_Y| \geq  \frac{1}{2}\int_{B_{\leq R}(\mu_X)}dP_X-dP_Y + \frac{1}{2} \int_{B_{\leq R}(\mu_Y) }dP_Y-dP_X \geq 1-\epsilon. 
\end{align}
\end{proof}

\begin{proof}[Proof of bound on $T_{\mathrm{upper}}$ in Lemma~\ref{corr:wassersteinlower}]
By Lemma \ref{lemma:convolve_subgaussian}, $p^i_t$ is sub-Gaussian with variance proxy $2\sigma^2$ for all $t \geq 0$.
For $i \in \Send, j \in [K]- \Send$, $\|\mu^i_t-\mu^j_t\| > 3\sigma \sqrt{8d \ln 6+ 8 \ln 4/\epsilon^2}$ implies $\TV(p^i_t,p^j_t) \geq 1-\epsilon^2/4$ by Lemma \ref{lemma:subgaussian_to_tv}. 
\end{proof}

\subsection{Score difference bound for Gaussian mixtures}\label{app:nablaratiolemma}

Here we prove the following key ingredient in the proof of Theorem~\ref{thm::well_conditioned_gaussian_theorem}, in analogy to Lemma~\ref{diff_score_expo} in the proof of the master theorem:

\begin{restatable}{lemma}{nablaratiolemma}\label{nabla_ratio_bound}
For any nonempty $S \subset [K]$ and $j \in S$, we have
\begin{align}
&\mathbb{E}_{x \sim p^j_t}\left[ \left \|\nabla \ln p_t^S - \nabla \ln p_t^{[K]-S}\right\|^4\right] \lesssim  \frac{e^{-4t}}{\underline{\lambda}^4} \left[\left(\overline{\lambda}-\underline{\lambda}\right)^4(\overline{R}(0)^4+\overline{\lambda}^2d^2)+\overline{R}(0)^4\right].
\end{align} 
\end{restatable}

\noindent To prove this, we need an auxiliary result:

\begin{lemma}\label{diff_between_bounded_sing_psd_matrices}
Let $A,B \in \mathbb{R}^{d\times d}$ be two PSD matrices with singular values in $[\underline{\sigma},\overline{\sigma}]$. For any $v \in \mathbb{R}^d$, 
$$\|(A-B)v\| \leq  2(\overline{\sigma}-\underline{\sigma})\|v\|.$$
\end{lemma}
\begin{proof}
We subtract both $Av,Bv$ by $\underline{\sigma}I$ and apply the triangle inequality,
\begin{align}
    \|(A-B)v\| = \|(A-\underline{\sigma}I)v-(B-\underline{\sigma}I)v\|\leq \|(A-\underline{\sigma}I)v\|+\|(B-\underline{\sigma}I)v\| \leq 2(\overline{\sigma}-\underline{\sigma})\|v\|.
\end{align}
\end{proof}

\begin{proof}[Proof of Lemma~\ref{nabla_ratio_bound}]
We explicitly compute $\nabla \ln p_t^S$ and $\nabla \ln p_t^{[K]-S}$ and their difference,
\begin{align}
\nabla \ln p_t^S&=\sum_{i \in S}\frac{w_ip^i_t}{\sum_{j \in S}w_jp^j_t}\left(-(\Sigma_i^t)^{-1}(x-\mu_i(t))\right)\\
\nabla \ln p_t^{[K]-S}&=\sum_{i \in [K]-S}\frac{w_ip^i_t}{\sum_{j \in S}w_jp^j_t}\left(-(\Sigma_i^t)^{-1}(x-\mu_i(t))\right)\\
\nabla \ln p_t^S - \nabla \ln p_t^{[K]-S}&=-\left(\sum_{i \in S}\frac{w_ip^i_t}{\sum_{j \in S}w_jp^j_t}(\Sigma_i^t)^{-1}-\sum_{i \in [K]-S}\frac{w_ip^i_t}{\sum_{j \in [K]-S}w_jp^j_t}(\Sigma_i^t)^{-1}\right)x\\
&+\left(\sum_{i \in S}\frac{w_ip^i_t}{\sum_{j \in S}w_jp^j_t}(\Sigma_i^t)^{-1}\mu_i(t)-\sum_{i \in [K]-S}\frac{w_ip^i_t}{\sum_{j \in [K]-S}w_jp^j_t}(\Sigma_i^t)^{-1}\mu_i(t)\right).
\end{align}
Both $\sum_{i \in S}\frac{w_ip^i_t}{\sum_{i \in S}w_ip^i_t}(\Sigma_i^t)^{-1},\sum_{i \in [K]-S}\frac{w_ip^i_t}{\sum_{i \in S}w_ip^i_t}(\Sigma_i^t)^{-1}$ are PSD matrices with singular values in $[1/\sigmamax(t),1/\sigmamin(t)]$. Thus, by Lemma \ref{diff_between_bounded_sing_psd_matrices}, we can bound the first term in the difference in terms of the norm of $x$, 
\begin{align}
\left\|\left(\sum_{i \in S}\frac{w_ip^i_t}{\sum_{j \in S}w_jp^j_t}(\Sigma_i^t)^{-1}-\sum_{i \in [K]-S}\frac{w_ip^i_t}{\sum_{j \in [K]-S}w_jp^j_t}(\Sigma_i^t)^{-1}\right)x\right\| \leq\left(1/\sigmamin(t)-1/\sigmamax(t)\right)\|x\|.
\end{align}
By the triangle inequality, we can bound the second term with the singular values as well,
\begin{align}
\left\|\sum_{i \in S}\frac{w_ip^i_t}{\sum_{j \in S}w_jp^j_t}(\Sigma_i^t)^{-1}\mu_i(t)-\sum_{i \in [K]-S}\frac{w_ip^i_t}{\sum_{j \in [K]-S}w_jp^j_t}(\Sigma_i^t)^{-1}\mu_i(t)\right\| \lesssim \overline{R}(t)/\sigmamin(t).
\end{align}
We can decompose $\mathbb{E}_{x \sim p_t^{\{j\}}}\|x\|^4$ into $\overline{R}(t),\sigmamax(t),d$ with the triangle inequality,
\begin{align}
\mathbb{E}_{x \sim p^j_t}\|x\|^4 \lesssim \overline{R}(t)^4+\sigmamax(t)^2\mathbb{E}_{x \sim p^j_t}\|\Sigma_i(t)^{-1/2}(x-\mu_i(t))\|^4 \lesssim \overline{R}(t)^4+\sigmamax(t)^2d^2.
\end{align}
Combining these inequalities, we obtain
\begin{align}
\mathbb{E}_{x \sim p^j_t}\left[ \left \|\nabla \ln p_t^S - \nabla \ln p_t^{[K]-S}\right\|^4\right] &\lesssim \left(1/\sigmamin(t)-1/\sigmamax(t)\right)^4(\overline{R}(t)^4+\sigmamax(t)^2d^2)+\overline{R}(t)^4/\sigmamin(t)^4\\
&\leq \frac{e^{-4t}}{\underline{\lambda}^4} \left[\left(\overline{\lambda}-\underline{\lambda}\right)^4(\overline{R}(0)^4+\overline{\lambda}^2d^2)+\overline{R}(0)^4\right].
\end{align}
\end{proof}

\subsection{Ratio bound for Gaussian mixtures}\label{app:ratioboundgaussian}

Here we prove the other key ingredient in the proof of Theorem~\ref{thm::well_conditioned_gaussian_theorem}, in analogy to Lemma~\ref{ratio_inequality} in the proof of the master theorem:

\begin{restatable}{lemma}{ratioboundgaussian}\label{ratioboundgaussian}
For any $S \subset [K]$ and $j \in S$, we have 
\begin{align}
  \mathbb{E}_{x \sim p^j_t}\left( \frac{\sum_{i \in [K]-S} w_i p^i_t}{\sum_{i\in[K]} w_i p^i_t}\right)^4 \lesssim K \overline{W}\exp \left\{-e^{-2t}\Delta(\Send)^2/(8\overline{\lambda})\right\}.
\end{align}
\end{restatable}

\noindent We will need the following helper lemmas:

\begin{lemma}\label{hellinger_gaussian}
(p.~51 of \cite{Pardo:996837}) Let $P \sim \mathcal{N}(\mu_P,\Sigma_P)$ and $Q \sim \mathcal{N}(\mu_Q,\Sigma_Q)$. Then, 
\[\Hellinger^2(P,Q)=2-2\frac{|\Sigma_P|^{1/4}|\Sigma_Q|^{1/4}}{\left|\frac{\Sigma_P+\Sigma_Q}{2}\right|^{1/2}} \exp \left\{-\frac{1}{8}(\mu_P-\mu_Q)^\top \left[\frac{\Sigma_P+\Sigma_Q}{2}\right]^{-1}(\mu_P-\mu_Q)\right\}.\]
\end{lemma}
\begin{lemma}\label{lemm:am_gm_det}
For positive semi-definite $\Sigma_i,\Sigma_j$, we have an AM-GM-style inequality for their determinants,
\begin{align*}
|\Sigma_i| \cdot |\Sigma_j| \leq \left|\frac{\Sigma_i+\Sigma_j}{2}\right|^2.
\end{align*}
\end{lemma}
\begin{proof}
It suffices to show $1\leq \left|\frac{1+\Sigma_i^{-1/2}\Sigma_j\Sigma_i^{-1/2}}{2}\right|\cdot \left|\frac{1+\Sigma_j^{-1/2}\Sigma_i\Sigma_j^{-1/2}}{2}\right|$. Both $(\Sigma_i^{-1/2}\Sigma_j\Sigma_i^{-1/2})^{-1}=\Sigma_i^{1/2}\Sigma_j^{-1}\Sigma_i^{1/2}$ and $\Sigma_j^{-1/2}\Sigma_i\Sigma_j^{-1/2}$ have the same spectrum and the same algebraic multiplicities. They are also positive semi-definite, which means the geometric multiplicities of their eigenvalues sum to $d$. Thus, we can conclude that both matrices have the same multiset of eigenvalues. Letting $\lambda_1,\lambda_2,\dots,\lambda_d>0$ be the eigenvalues of $(\Sigma_i^{-1/2}\Sigma_j\Sigma_i^{-1/2})^{-1},\Sigma_j^{-1/2}\Sigma_i\Sigma_j^{-1/2}$, the right-hand side can be bounded by
\begin{align*}
\left|\frac{1+\Sigma_i^{-1/2}\Sigma_j\Sigma_i^{-1/2}}{2}\right|\cdot \left|\frac{1+\Sigma_j^{-1/2}\Sigma_i\Sigma_j^{-1/2}}{2}\right| \geq \prod_{i=1}^d \left(\frac{1+1/\lambda_i}{2}\right) \left(\frac{1+\lambda_i}{2}\right) =  \prod_{i=1}^d \frac{2+1/\lambda_i+\lambda_i}{4} \geq 1.
\end{align*}
\end{proof}

\begin{proof}[Proof of Lemma~\ref{ratioboundgaussian}]
 Because $\mathbb{E}_{x \sim p^j_t}\left( \frac{\sum_{i \in [K]-S} w_i p^i_t}{\sum_{i\in[K]} w_i p^i_t}\right)^4   \leq \sum_{\ell \in [K]-S} \mathbb{E}_{x \sim p^j_t}\left[\frac{ w_\ell p^\ell_t}{\sum_{i\in[K]} w_i p^i_t}\right]$, it suffices to bound $\mathbb{E}_{x \sim p^j_t}\left[\frac{w_\ell p_\ell^t}{w_\ell p_\ell^t + w_jp^j_t}\right]$ for any $\ell \in [K]-S$. Using the Hellinger distance bound in Lemma \ref{ratio_inequality} and the computations in Lemmas~\ref{hellinger_gaussian} and ~\ref{lemm:am_gm_det}, we have 
\begin{align}
\mathbb{E}_{x \sim p^j_t}\left[\frac{ w_\ell p_\ell^t}{w_\ell p_\ell^t + w_jp^j_t}\right] &\leq \overline{W} \frac{|\Sigma_\ell(t)|^{1/4}|\Sigma_j(t)|^{1/4}}{\left|\frac{\Sigma_\ell(t)+\Sigma_j(t)}{2}\right|^{1/2}} \exp \left\{-\frac{e^{-2t}}{8}(\mu_\ell-\mu_j)^\top \left[\frac{\Sigma_\ell(t)+\Sigma_j(t)}{2}\right]^{-1}(\mu_\ell-\mu_j)\right\}\\
&\lesssim\overline{W} \exp \left\{-e^{-2t}\Delta(\Send)/(8\overline{\lambda})\right\}.
\end{align}
\end{proof}

\subsection{Proof of Theorem~\ref{thm::well_conditioned_gaussian_theorem}}\label{app:gauss_thm}
\mastersgaussiantheorem*
\begin{proof}
As in the proof of Theorem~\ref{masters_theorem}, we apply the data processing inequality to obtain
\begin{align}
    \TV(\modrevlaw{\Sinit}{\wh{T}}{},p^{\Send}) \leq \TV(\modrevpath{\Sinit}{\wh{T}}, \modrevpath{\Send}{\wh{T}})+\TV(\modrevpath{\Send}{\wh{T}}, \pathrev{\Send}_{\wh{T}}).
\end{align}
We begin with $\TV(p^{S_{\mathrm{init}}  }_{\widehat{T}},p^{\Send}_{\widehat{T}})$. By Lemma \ref{mixture_tv_bound}, it suffices to show for any $i \in S_{\mathrm{init}}, j \in  \Send$, $\TV(p^i_{\widehat{T}},p^j_{\widehat{T}}) \leq \epsilon$. To control this quantity, we use Pinsker's inequality to write in terms of $\KL$ and the $\KL$ formula for two Gaussians, and further bound the determinant and trace in terms of $\underline{\lambda},\overline{\lambda}$.
\begin{align}
\TV(p^{i  }_{\widehat{T}},p^{j}_{\widehat{T}}) &\leq \sqrt{\KL(p^{i  }_{\widehat{T}},p^{j}_{\widehat{T}})}\\
&=\sqrt{\ln \frac{|\Sigma^j(\widehat{T})|}{|\Sigma^i(\widehat{T})|}+d\left[\frac{1}{d}\mathrm{tr}(\Sigma_j^{-1}\Sigma_i)-1\right]+(\mu_i(\widehat{T})-\mu_j(\widehat{T}))^\top \Sigma^j(\widehat{T})^{-1}(\mu_i(\widehat{T})-\mu_j(\widehat{T}))}\\
&\leq \sqrt{d\left[\ln \left(\frac{e^{-2\widehat{T}}\overline{\lambda}+1-e^{-2\widehat{T}}}{e^{-2\widehat{T}}\underline{\lambda}+1-e^{-2\widehat{T}}}\right) +\frac{e^{-2\widehat{T}}\overline{\lambda}+1-e^{-2\widehat{T}}}{e^{-2\widehat{T}}\underline{\lambda}+1-e^{-2\widehat{T}}}-1\right]+\frac{1}{\underline{\lambda}}\|\mu_i-\mu_j\|^2e^{-2\widehat{T}}}
\end{align}
We now use the inequality $\ln(x) \leq x-1$ and note $\frac{e^{-2t}\overline{\lambda}+1-e^{-2t}}{e^{-2t}\underline{\lambda}+1-e^{-2t}}-1 \leq e^{-2t}\frac{\overline{\lambda}-\underline{\lambda}}{\underline{\lambda}}$, 
\begin{align}
\TV(p^{i  }_{\widehat{T}},p^{j}_{\widehat{T}})  \leq  \sqrt{2e^{-2\widehat{T}}d(\overline{\lambda}-\underline{\lambda})/\underline{\lambda}+\frac{1}{\underline{\lambda}}\|\mu_i-\mu_j\|^2e^{-2\widehat{T}}}\leq \epsilon
\end{align}
Now we bound $\TV(\modrevpath{\Send}{\wh{T}}, \pathrev{\Send}_{\wh{T}}).$ Following the main Cauchy-Schwarz split in Theorem~\ref{masters_theorem}, we can  apply Lemmas~\ref{nabla_ratio_bound} and ~\ref{ratioboundgaussian} to control the score error for $t \in [0,\wh{T}]$,
\begin{align}
&\mathbb{E}\left[\|\nabla \ln p^{\Send}_t(\overline{X}_t^{\Send}) -\nabla \ln p^{[K]}_t(\overline{X}_t^{\Send})\|^2\right]  \\
&\lesssim e^{-2t}\frac{\sqrt{K\overline{W}}\left[\left(\overline{\lambda}-\underline{\lambda}\right)^2(\overline{R}(0)^2+\overline{\lambda}d)+\overline{R}(0)^2\right]}{\underline{\lambda}^2} \exp \left\{-e^{-2t}\underline{\lambda}\Delta(\Send)^2/(16\overline{\lambda})\right\}.
\end{align}
The integral from $0$ to $\wh{T}$ is 
\begin{align}
&\int_0^{\wh{T}}\mathbb{E}\left[\|\nabla \ln p^{\Send}_t(\overline{X}_t^{\Send}) -\nabla \ln p^{[K]}_t(\overline{X}_t^{\Send})\|^2\right] dt \\
&\lesssim \frac{\sqrt{K\overline{W}}\overline{\lambda}\left[\left(\overline{\lambda}-\underline{\lambda}\right)^2(\overline{R}(0)^2+\overline{\lambda}d)+\overline{R}(0)^2\right]}{\underline{\lambda}^2\Delta(\Send)^2} \exp \left\{-e^{-2T_{\mathrm {upper}}}\Delta(\Send)^2/(16\overline{\lambda})\right\}\lesssim \epsilon^2.
\end{align}
\end{proof}
\subsection{Proof of Theorem~\ref{thm:hierarchy_example}}\label{app:hierarchy}
\thmhierarchy*
\begin{proof}
Using the notation from Example \ref{ex:identity_gaussians}, let 
\begin{align}
T_{\rm lower}^{j}&=\ln w(f(u_j),f(u_{j+1}))+\ln 1/\epsilon  \\
T_{\rm upper}^{j}&=\ln \Delta(f(u_{j+1})) -\ln 4-\frac{1}{2}\ln \ln \frac{R^2}{ \epsilon^2\Delta(f(u_{j+1}))^2 }.
\end{align}
It suffices to show that for a sufficiently large $k$, for all $j \leq k$, we have both $T_{\rm upper}^{j}-T_{\rm lower}^{j} > 0$ and $T_{\rm lower}^{j+1}-T_{\rm upper}^{j} > 0$. By our definition of the mixture tree, we know 
\begin{align*}
w(f(u_j),f(u_{j+1})) &\in \left[(1-\delta)\frac{R}{2^{(H'-j)^2}},(1+\delta)\frac{R}{2^{(H'-j)^2}}\right]  \\
\Delta(f(u_{j+1})) &\in \left[(1-\delta)\frac{R}{2^{(H'-j-1)^2}},(1+\delta)\frac{R}{2^{(H'-j-1)^2}}\right].
\end{align*}
$T_{\rm lower}^{j+1}-T_{\rm upper}^{j}>0$ follows from 
\begin{align*}
T^{j+1}_{\rm lower} = \ln \left[(1-\delta) \frac{R}{2^{(H'-j-1)^2}} \right] + \ln \frac{1}{\epsilon}\geq \ln \left[(1+\delta) \frac{R}{2^{(H'-j-1)^2}} \right] -\frac{1}{2}\ln \ln \frac{R^2}{ \epsilon^2\Delta(f(u_{j+1}))^2 } \geq T^j_{\rm upper}.
\end{align*}
for sufficiently small $\epsilon$. We have $T_{\rm upper}^{j}-T_{\rm lower}^{j} > 0$ if 
\begin{align}
&\ln\left[(1+\delta)\frac{R}{2^{(H'-j)^2}}\right] + \ln(1/\epsilon) \leq \ln\left[(1-\delta)\frac{R}{2^{(H'-j-1)^2}}\right]-\ln 4 - \frac{1}{2} \ln \ln \left[\frac{R^2}{\epsilon^2\left((1-\delta)\frac{R}{2^{(H'-j-1)^2}}\right)^2}\right] \\
&\ln \frac{1+\delta}{1-\delta}+ \ln \frac{1}{\epsilon} + \ln 4 + \frac{1}{2}\ln \left[2(H'-j-1)^2\ln 2 -2 \ln (1-\delta)\epsilon \right]\leq (2(H'-j)-1) \ln 2.
\end{align}
This is true for sufficiently small $j$ and large $H'$. 
\end{proof}

\section{Other concrete calculations for critical windows}

Here we include two additional calculations, one for Gaussian mixtures with imbalanced mixing weights, and one in a dictionary learning setting, that were deferred to the appendix due to space constraints.

\subsection{Dependence of $\Tupper,\Tlower$ on mixing weights}
\label{app:weights}

Consider the scenario of two Gaussians with identity covariance.
\begin{example}\label{example:weighted} (Two Gaussians with identity covariance) Let $K=2$, $p^1_0 = \mathcal{N}(\mu,\mathrm{Id})$, $p^2_0 = \mathcal{N}(-\mu,\mathrm{Id})$. Then, focusing on component $1$ we have 
\begin{align}
T_{\rm one} &= \ln \|\mu\| -\ln 2-\frac{1}{2}\ln \ln \frac{\sqrt{2w_2/w_1}}{4 \epsilon^2} \\
T_{\rm all} &= \ln \|\mu\| +\ln 2 + \ln 1/\epsilon
\end{align}
When $\widehat{T} \leq T_{\rm one}$, then $\TV(\modrevlaw{\{1\}}{\wh{T}}{},p^{\{1\}}) \lesssim \epsilon$. When $\widehat{T} \geq T_{\rm all}$, $\TV(\modrevlaw{\{1\}}{\wh{T}}{},p^{\{1,2\}})\lesssim \epsilon$. We can see that as $w_2$ increases, the cutoff $T_{\rm one}$ becomes smaller, though the amount by which it decreases only scales at $O(\ln \ln w_2/w_1)$. 
\end{example}

\subsection{Sparse dictionary example}\label{sec:dictionary}
Now we consider a dictionary learning setting, in which classes are described by subsets of nearly-orthogonal feature vectors. Consider a set of $F=\{f_1,f_2,\dots,f_n\}$ unit vectors, such that for all distinct $i,j$, $\mathrm{cov}(f_i,f_j) \leq \delta$. Fix some large $R = \Omega(d)$. Consider the families of random variables $\mathcal{Y}_\ell = \{Y \in \mathbb{R}^\ell: \mathbb{E}[Y] =0,Y \in \subG_\ell(\sigma^2)\}$. We define scalar random variables $Y_{S,i} \in \mathcal{Y}_1$ for $S \subset F$ and $i \in [n]$, that represent the scaling along each feature vector, and $Y_{S} \in \mathcal{Y}_d$, which represents variation not along the features. Classes are subsets $S \subset F$ of cardinality $|S| \leq \tilde{S}$, such that a sample $X \sim p^S_t$ has the distribution of $\sum_{i \in S} (Y_{S,i}+R) f_i+Y_S$. We let the Wasserstein-$2$ distance between any centered classes be less than $\Upsilon$. We can characterize the $T_{\rm lower},T_{\rm upper}$ in terms of the Hamming distances $H$ between classes. We define $\overline{H}(S,S'):=\max_{i\in S,j \in S'}H(i,j)$ and $\underline{H}(S)=\min_{\ell \in S,j \in [K]-S} H(i,j).$ By parameter setting with Corollary \ref{corr:wassersteinlower}, we can write $T_{\rm lower},T_{\rm upper}$ in terms of Hamming distances between classes.
\begin{restatable}{corollary}{sparsecutoffsetting}\label{corr:sparse_cutoff_setting}
    We have that $T_{\rm lower}(\epsilon) \le 3 \vee \Bigl\{ \ln \frac{1}{\epsilon}+\frac{1}{2} \ln 2 + \ln (R\sqrt{\overline{H}(S_{\rm init},S_{\rm end})+d^2\delta}+\Upsilon)\Bigr\}$ and $T_{\rm upper}(\epsilon) \ge \ln \left(R\sqrt{\underline{H}(S_{\rm end})-d^2\delta}\right) - \ln (\sigma \sqrt{\tilde{S}+1})- \ln\sqrt{8d \ln 6 + 8\ln 4/\epsilon^2}- \ln 3-\frac{1}{2}\ln 8$.
\end{restatable}

\begin{proof}
We show that $\|\mu_i-\mu_j\|$ is only slightly differs from a constant factor from the Hamming distance, 
\begin{align*}
\|\mu_i-\mu_j\|^2 =R^2\left\|\sum_{\ell \in i \backslash j}f_\ell - \sum_{\ell \in j \backslash i}f_\ell\right\|^2 \in \left[R^2(H(i,j)-d^2\delta),R^2(H(i,j)+d^2\delta)\right]
\end{align*}
This completes $T_{\rm lower}$. For $T_{\rm upper}$, we also need to upper bound the variance proxies for each component. Letting $X \sim \sum_{i \in S} Y_{S,i}f_i+Y_S$, we can compute for all $u \in \mathbb{S}^{d-1}$ the expectation $\mathbb{E}[\exp(su^\top X)]$, 
\begin{align*}
\mathbb{E}[\exp(su^\top X)]&=\mathbb{E}[\exp(su^\top X)] = \mathbb{E}\left(\exp\left(su^\top Y_S\right)\right) \prod_{i \in S} \mathbb{E}\left(\exp\left(su^\top f_i Y_i\right)\right)  \\
&\leq \exp\left(s^2 \sigma^2 /2 \right)   \prod_{i \in S} \exp\left(s^2 \sigma^2 (u^\top f_i)^2/2 \right)  \\
&\leq  \exp\left(\frac{s^2\sigma^2(|S|+1)}{2}\right) \leq  \exp\left(\frac{s^2\sigma^2(\tilde{S}+1)}{2}\right) . 
\end{align*}
Thus $X \in \subG_d(\sigma(\tilde{S}+1))$. 
\end{proof}

\end{document}